\documentclass{article}

\usepackage{amsfonts}
\usepackage{amsthm}
\usepackage[cmex10]{amsmath}

\usepackage{import}
\usepackage{graphicx}
\usepackage{xspace}                           

\newtheorem{theorem}{{\bf Theorem}}
\newtheorem{lemma}[theorem]{{\bf Lemma}}
\newtheorem{corollary}[theorem]{{\bf Corollary}}

\newtheorem{definition}[theorem]{{\bf Definition}}
\newtheorem{proposition}[theorem]{{\bf Proposition}}

\newcommand{\prob}[1]{\Pr\left(#1\right)}
\newcommand{\expect}[1]{\mathbf{E}\left[#1\right]}

\newcommand{\zerobit}{0\nobreakdash-bit\xspace}

\newcommand{\onebits}{1\nobreakdash-bits\xspace}
\newcommand{\zerobits}{0\nobreakdash-bits\xspace}

\newcommand{\trans}[1]{{#1}^{\ensuremath{\mathsf{T}}}}    
\DeclareMathOperator{\diag}{diag}                         
\DeclareMathOperator{\poly}{poly}

\newcommand{\ab}{\hspace{0.125em}}                        

\newcommand{\ie}{\hbox{i.\ab e.}\xspace}                  
\newcommand{\etal}{et al.\ }

\usepackage{tikz}
\usetikzlibrary{arrows,patterns,plotmarks,shapes,snakes,er,3d,automata,backgrounds,topaths,trees,petri,mindmap}
\usepackage{algorithm,algorithmic}
\floatname{algorithm}{}

\newcommand{\selpres}{\text{{\sc SelPres}$_{\sigma,\delta,k}$}\xspace}
\newcommand{\leadingones}{\text{\sc LeadingOnes}\xspace}

\newcommand{\onemax}{\text{\sc OneMax}\xspace}

\begin{document}
\title{On the Impact of Mutation-Selection Balance\\
       on the Runtime of Evolutionary Algorithms\thanks{This work was supported by the EPSRC under grant
        no. EP/D052785/1, and by Deutsche Forschungsgemeinschaft (DFG) 
        under grant no. WI~3552/1-1.}}
\author{Per Kristian Lehre\thanks{Per Kristian Lehre is with 
DTU Informatics, Technical
  University of Denmark, 2800 Kongens Lyngby, Denmark. 
  (e-mail: pkle@imm.dtu.dk). } \and Xin Yao\thanks{Xin Yao is with The Centre of
        Excellence for Research in Computational Intelligence and 
        Applications (CERCIA), School of Computer Science, 
        The University of Birmingham, Edgbaston, Birmingham B15 2TT, UK        
        (e-mail: x.yao@cs.bham.ac.uk)}}

\maketitle

\begin{abstract}
  The interplay between mutation and selection
  plays a fundamental role in the behaviour of evolutionary
  algorithms (EAs). However, this interplay is still not completely
  understood.  This paper presents a rigorous runtime analysis of a
  non-elitist population-based EA that uses the linear ranking
  selection mechanism. The analysis focuses on how the balance between
  parameter $\eta$, controlling the selection pressure in linear
  ranking, and parameter $\chi$ controlling the bit-wise
  mutation rate, impacts the runtime of the algorithm.
  The results point out situations where a correct balance between
  selection pressure and mutation rate is essential for finding the
  optimal solution in polynomial time. In particular, it is shown that
  there exist fitness functions which can only be solved in polynomial
  time if the ratio between parameters $\eta$ and $\chi$ is within a
  narrow critical interval, and where a small change in this ratio can
  increase the runtime exponentially. Furthermore, it is shown
  quantitatively how the appropriate parameter choice depends on the
  characteristics of the fitness function.
  In addition to the original results on the runtime of EAs, this paper also 
  introduces a very useful analytical tool, \ie, multi-type branching 
  processes, to the runtime analysis of non-elitist population-based EAs.
\end{abstract}

\section{Introduction}
Evolutionary algorithms (EAs) have been applied successfully to many
optimisation problems \cite{Sarker2002EvOpt}. However, despite several
decades of research, many fundamental questions about their behaviour
remain open. One of the central questions regarding EAs is to
understand the interplay between the selection mechanism and the
genetic operators. Several authors have suggested that EAs must find a
balance between maintaining a sufficiently diverse population to
\emph{explore} new parts of the search space, and at the same time
\emph{exploit} the currently best found solutions by focusing the
search in this direction
\cite{Eiben1998Expl,Whitley1989GENITOR,Goldberg1991Selection}.  In
fact, the trade-off between exploration and exploitation has been a
common theme not only in evolutionary computation, but also in
operations research and artificial intelligence in general. However,
few theoretical studies actually exist that explain how to define such
trade-off quantitatively and how to achieve it. Our paper can be
regarded as one of the first rigorous runtime analyses of EAs that
addresses the interaction between exploration, driven by mutation, and
exploitation, driven by selection.

Much research has focused on finding measures to quantify the
selection pressure in selection mechanisms --- without taking into
account the genetic operators --- and subsequently on investigating
how EA parameters influence these measures
\cite{Goldberg1991Selection,Back1994Selection,Blickle1996Selection,SchlierkampVoosen1993Breeder,CantuPaz2002Order}.
One such measure, called the \emph{take-over time}, considers the
behaviour of an evolutionary process consisting only of the selection
step, and no crossover or mutation operators
\cite{Goldberg1991Selection,Back1994Selection}. Subsequent populations
are produced by selecting individuals from the previous generation,
keeping at least one copy of the fittest individual.  Hence, the
population will after a certain number of generations only contain
those individuals that were fittest in the initial population, and
this time is called the take-over time. A short take-over time
corresponds to a high selection pressure. Other measures of selection
pressure consider properties of the distribution of fitness values in
a population that is obtained by a single application of the selection
mechanism to a population with normally distributed fitness values.
One of these properties is the \emph{selection intensity}, which is
the difference between the average population fitness before and after
selection \cite{SchlierkampVoosen1993Breeder}. Other properties are
\emph{loss of diversity} \cite{Blickle1996Selection,
  Motoki2002Selection} and higher order cumulants of the fitness
distribution \cite{CantuPaz2002Order}.

To completely understand the role of selection mechanisms, it is
necessary to also take into account their interplay with the genetic
operators.  There exist few rigorous studies of selection mechanisms
when used in combination with genetic operators.  Happ \etal
considered fitness proportionate selection, which is one of the first
selection mechanisms to be employed in evolutionary algorithms
\cite{Happ2008Selection}.  Early research in evolutionary computation
pointed out that this selection mechanism suffers from various
deficiencies, including population stagnation due to low selective
pressure \cite{Whitley1989GENITOR}. Indeed, the results by Happ \etal
show that variants of the RLS and the (1+1) EA that use
fitness-proportional selection have exponential runtime on the class
of linear functions \cite{Happ2008Selection}. Their analysis was
limited to single-individual based EAs. Neumann \etal showed that even
with a population-based EA, the \onemax problem cannot be optimised in
polynomial time with fitness proportional selection
\cite{Neumann2009FitProp}.  However, they pointed out that polynomial
runtime can be achieved by scaling the fitness function. Witt also
studied a population-based algorithm with fitness proportionate
selection, however with the objective to study the role of populations
\cite{Witt2008Population}.  Chen \etal analysed the ($N$+$N$)~EA to
compare its runtimes with truncation selection, linear ranking
selection and binary tournament selection on the \leadingones\ and
\onemax\ problems \cite{Chen2008Selection}. They found the expected
runtime on these fitness functions to be the same for all three
selection mechanisms. None of the results above show how the balance
between the selection pressure and mutation rate impacts the runtime.

This paper analyses rigorously a non-elitist, population based EA that
uses linear ranking selection and bit-wise mutation. The main
contributions are an analysis of situations where the
mutation-selection balance has an exponentially large impact on the
runtime, and new techniques based on branching processes for analysing
non-elitist population based EAs.  The paper is based on preliminary
work reported in \cite{Lehre2009FOGA}, which contained the first
rigorous runtime analysis of a non-elitist, population based EA with
stochastic selection.  This paper significantly extends this early
work. In addition to strengthening the main result, simplifying
several proofs and proving a conjecture, we have added a completely
new section that introduces multi-type branching processes as an
analytical tool for studying the runtime of EAs.

\subsection{Notation and Preliminaries}
The following notation will be used in the rest of this paper.  The
length of a bitstring $x$ is denoted $\ell(x)$.  The $i$-th bit, $1\leq
i\leq \ell(x),$ of a bitstring $x$ is denoted $x_i$.  The
concatenation of two bitstrings $x$ and $y$ is denoted by $x\cdot y$ and
$xy$. Given a bitstring $x$, the notation $x[i,j]$,
where $1\leq i<j\leq \ell(x)$, denotes the substring $x_ix_{i+1}\cdots
x_j$.  For any bitstring $x$, define
$\|x\|:=\sum_{i=1}^{\ell(x)}x_i/\ell(x)$, \ie the fraction of
\onebits in the bitstring. We say that an event holds with
\emph{overwhelmingly high probability (w.o.p.)} with respect to a
parameter $n$, if the probability of the event is bounded from 
below by $1-e^{-\Omega(n)}$.

In contrast to classical algorithms, the runtime of EAs is usually
measured in terms of the number of evaluations of the fitness
function, and not the number of basic operations. For a given function
and algorithm, the \emph{expected runtime} is defined as the mean
number of fitness function evaluations until the optimum is evaluated
for the first time. The runtime on a class of fitness functions is
defined as the supremum of the expected runtimes of the functions in
the class \cite{DJW02Analysis}. The variable name $\tau$ will be used
to denote the runtime in terms of number of generations of the EA. In
the case of EAs that are initialised with a population of $\lambda$
individuals, and which in each generation produce $\lambda$
offspring, variable $\tau$ can be related to the runtime $T$ by
$\lambda(\tau-1)\leq T\leq \lambda\tau$.

\section{Definitions}
\subsection{Linear Ranking Selection}
In ranking selection, individuals are selected according to their
fitness rank in the population. A ranking selection mechanism
is uniquely defined by the probabilities $p_i$ of selecting an 
individual ranked $i$, for all ranks $i$
\cite{Blickle1996Selection}. For mathematical convenience, an
alternative definition due to Goldberg and Deb
\cite{Goldberg1991Selection} is adopted, in which a function
$\alpha:[0,1]\rightarrow\mathbb{R}$ is considered a ranking
function if it is non-increasing, and satisfies the 
following two conditions
\begin{enumerate}
\item $\alpha(x)\geq 0$, \text{ and}
\item $\int_0^1 \alpha(y) dy = 1$.
\end{enumerate}
Individuals are ranked from 0 to 1, with the best individual ranked 0,
and the worst individual ranked 1. For a given ranking function
$\alpha$, the integral $\beta(x,y) := \int_x^y\alpha(z)dz$ gives the
probability of selecting an individual with rank between $x$ and $y$.
By defining the linearly decreasing ranking function $\alpha(x) :=
\eta - cx$, where $\eta$ and $c$ are parameters, one obtains
\emph{linear ranking selection}.  The first condition
implies that $\eta\geq c\geq 0$, and the second condition implies that
$c=2(\eta-1)$. Hence, for linear ranking selection, we have
\begin{align}
  \alpha(x) & := \eta(1-2x)+2x\label{eq:alpha},\text{ and}\\
  \beta(x)  & := \beta(0,x)  = x(\eta(1-x) + x).\label{eq:beta}
\end{align}
Note that since $\alpha$ is non-increasing, \ie, $\alpha'(x)\leq 0$,
we must have $\eta\geq 1$. Also, the special case $\alpha(1)\geq 0$ of
the first condition implies that $\eta\leq 2$.  The selection
pressure, measured in terms of the take-over time, is uniquely 
given by, and monotonically decreasing in
the parameter $\eta$ \cite{Goldberg1991Selection}.  The weakest
selection pressure is obtained for $\eta=1$, where selection is
uniform over the population, and the highest selection pressure is
obtained for $\eta=2$. We therefore assume that $1<\eta\leq 2$.

\subsection{Evolutionary Algorithm}

\begin{algorithm}
  \caption{Linear Ranking EA \cite{Lehre2009FOGA}}
  \begin{algorithmic}[1]
    \STATE $t\leftarrow 0$.
    \FOR{$i=1$ to $\lambda$}
    \STATE Sample $x$ uniformly at random from $\{0,1\}^n$.
    \STATE $P_0(i) \leftarrow  x$.
    \ENDFOR
    \REPEAT
    \STATE Sort $P_t$ according to fitness $f$, such that\\
    \quad\quad $f(P_t(1))\geq f(P_t(2))\geq \cdots\geq f(P_t(\lambda))$.
    \FOR{$i=1$ to $\lambda$}
    \STATE Sample $r$ in $\{1,...,\lambda\}$ with $\prob{r\leq \gamma\lambda}=\beta(\gamma)$.
    \STATE $P_{t+1}(i) \leftarrow P_t(r)$.
    \STATE Flip each bit position in $P_{t+1}(i)$ with prob. $\chi/n$.
    \ENDFOR
    \STATE $t\leftarrow t+1$.
    \UNTIL{termination condition met.}
  \end{algorithmic}
  \label{fig:ea}
\end{algorithm}

We consider a population-based non-elitist EA
which uses linear ranking as selection mechanism. The crossover
operator will not be considered in this paper. The pseudo-code of the
algorithm is given above. After sampling the initial
population $P_0$ at random in lines 1 to 5, the algorithm enters its
main loop where the current population $P_t$ in generation $t$ is
sorted according to fitness, then the next population $P_{t+1}$ is
generated by independently selecting (line 9) and mutating (line 10)
individuals from the previous population $P_t$. The analysis of the
algorithm is based on the assumption that parameter $\chi$ is a
constant with respect to $n$.

Linear ranking selection is indicated in line 9, where for a given
selection pressure $\eta$, the cumulative probability of sampling
individuals with rank less than $\gamma\lambda$ is $\beta(\gamma)$. It
can be seen from the definition of the functions $\alpha$ and $\beta$,
that the upper bound $\beta(\gamma,\gamma+\delta)\leq
\delta\cdot\alpha(\gamma)$, holds for any $\gamma,\delta>0$ where
$\gamma+\delta\leq 1$. Hence, the expected number of times a uniformly
chosen individual ranked between $\gamma\lambda$ and
$(\gamma+\delta)\lambda$ is selected during one generation is upper
bounded by $(\lambda/\delta\lambda)\cdot
\beta(\gamma,\gamma+\delta)\leq \alpha(\gamma)$.  We leave the
implementation details of the sampling strategy unspecified, and
assume that the EA has access to some sampling mechanism which draws
samples perfectly according to $\beta$.

\subsection{Fitness Function}

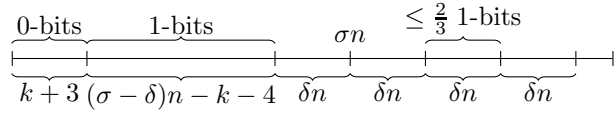
\begin{figure}
  \centering
  \begin{tikzpicture}
    \draw[|-|] (0cm,0cm) -- (8cm,0cm);
    \draw (1cm,-0.1cm) -- (1cm,0.1cm);
    \draw[snake=brace] (1cm,-0.2cm) -- node [below] {$k+3$}      (0cm,-0.2cm);

    \draw (3.5cm,-0.1cm) -- (3.5cm,0.1cm);
    \draw (4.5cm,-0.1cm) -- (4.5cm,0.1cm) node [above] {$\sigma n$};
    \draw (5.5cm,-0.1cm) -- (5.5cm,0.1cm);
    \draw (6.5cm,-0.1cm) -- (6.5cm,0.1cm);
    \draw (7.5cm,-0.1cm) -- (7.5cm,0.1cm);

    \draw[snake=brace] (3.5cm,-0.2cm) -- node [below] {$(\sigma-\delta)n-k-4$} (1cm,-0.2cm);
    \draw[snake=brace] (4.5cm,-0.2cm) -- node [below] {$\delta n$} (3.5cm,-0.2cm);
    \draw[snake=brace] (5.5cm,-0.2cm) -- node [below] {$\delta n$} (4.5cm,-0.2cm);
    \draw[snake=brace] (6.5cm,-0.2cm) -- node [below] {$\delta n$} (5.5cm,-0.2cm);
    \draw[snake=brace] (7.5cm,-0.2cm) -- node [below] {$\delta n$} (6.5cm,-0.2cm);

    \draw[snake=brace] (5.5cm,0.2cm) -- node [above] {$\leq \frac{2}{3}\; 1$-bits}  (6.5cm,0.2cm);
    \draw[snake=brace] (0cm,0.2cm) -- node [above] {0-bits} (1cm,0.2cm);
    \draw[snake=brace] (1cm,0.2cm) -- node [above] {1-bits} (3.5cm,0.2cm);
  \end{tikzpicture}
  \caption{\label{fig:fitness}Illustration of optimal search points \cite{Lehre2009FOGA}.}
\end{figure}

\begin{definition}
  For any constants $\sigma, \delta$, $0<\delta<\sigma<1-3\delta$,
  and integer $k\geq 1$, define the function  
  \begin{align*}
  \selpres(x) & :=
    \begin{cases}
      2n     & \text{if } x\in X_\sigma^*,\text{ and}\\
      \sum_{i=1}^n\prod_{j=1}^i x_j  & \text{otherwise,}
    \end{cases}
  \end{align*}
  where the set of optimal solutions $X_\sigma^*$ is defined to 
  contain all bitstrings $x\in\{0,1\}^n$ satisfying
  \begin{align*}
    \|x[1,k+3]\|& =0,\\
    \|x[k+4, (\sigma-\delta)n-1]\| & = 1, \text{ and}\\
    \|x[(\sigma+\delta)n,(\sigma+2\delta)n-1]\| & \leq 2/3.
  \end{align*}
\end{definition}

Except for the set of globally optimal solutions $X^*_\sigma$, the
fitness function takes the same values as the well known \leadingones\
fitness function, \ie the number of leading 1\nobreakdash-bits in the
bitstring. The form of the optimal search points, which is illustrated
in Fig.~\ref{fig:fitness}, depends on the three problem parameters
$\sigma, k$ and $\delta$. The $\delta$-parameter is needed for
technical reasons and can be set to any positive constant arbitrarily
close to 0. Hence, the globally optimal solutions have approximately
$\sigma n$ leading 1\nobreakdash-bits, except for $k+3$ leading
0\nobreakdash-bits. In addition, globally optimal search points must
have a short interval after the first $\sigma n$ bits which does not
contain too many 1\nobreakdash-bits.

\section{Main Result}
\begin{theorem}\label{thm:main-result}
  For any constant integer $k\geq 1$, let $T$ be the runtime of the 
  Linear Ranking EA with population
  size $n\leq \lambda\leq n^k$ with 
  a constant selection pressure of $\eta, {1<\eta\leq 2}$, and 
  bit-wise mutation rate $\chi/n$, for a constant $\chi>0$, on 
  function \selpres with parameters $\sigma$ and 
  $\delta$, where $0<\delta<\sigma<1-3\delta$.
  Let $\epsilon>0$ be any constant.
  \begin{enumerate}
  \item If $\eta < \exp(\chi(\sigma-\delta))-\epsilon$, then for some
    constant $c>0$,
    \begin{align*}
      \prob{T\geq e^{cn}}=1-e^{-\Omega(n)}.
    \end{align*}
  \item If $\eta = \exp(\chi\sigma)$, then
    \begin{align*}
      \prob{T\leq n^{k+4}} = 1-e^{-\Omega(n)}.
    \end{align*}
  \item If $\eta > (2\exp(\chi(\sigma+3\delta))-1)/(1-\delta)$, then
    \begin{align*}
      \expect{T} = e^{\Omega(n)}.
    \end{align*}
  \end{enumerate}
\end{theorem}
\begin{proof}
  The theorem follows from Theorem~\ref{thm:succprob-corr-selpres},
  Theorem~\ref{thm:runtime-high-selpres},
  and Corollary~\ref{cor:runtime-low-selpres}.
\end{proof}

Theorem~\ref{thm:main-result} describes how the runtime of
the Linear Ranking EA on fitness function \selpres depends on the main
problem parameters $\sigma$ and $k$, the mutation rate $\chi$ and the
selection pressure $\eta$. The theorem is illustrated in
Figure~\ref{fig:finalresult} for problem parameter $\sigma=1/2$.  Each
point in the grey area indicates that for the corresponding values of
mutation rate $\chi$ and selection pressure $\eta$, the EA has
either expected exponential runtime or exponential runtime with
overwhelming probability (\ie is highly inefficient). The
thick line indicates values of $\chi$ and $\eta$ where the 
runtime of the EA is polynomial with overwhelmingly high
probability (\ie is efficient). The runtime in the white 
regions is not analysed. 

The theorem and the figure indicate that setting one of the two
parameters of the algorithm (\ie  $\eta$ or $\chi$) independently of
the other parameter is insufficient to guarantee polynomial
runtime. For example, setting the selection pressure parameter to
$\eta:=3/2$ only yields polynomial runtime for certain settings of the
mutation rate parameter $\chi$, while it leads to exponential 
runtime for other settings of the mutation rate parameter.  Hence, it
is rather the balance between the mutation rate $\chi$ and the
selection pressure $\eta$, \ie  the \emph{mutation-selection
  balance}, that determines the runtime for the Linear
Ranking EA on this problem. More specifically, a too high setting of
the selection pressure parameter $\eta$ can be compensated by
increasing the mutation rate parameter $\chi$. Conversely, a too low
parameter setting for the mutation rate $\chi$ can be compensated by
decreasing the selection pressure parameter $\eta$.
Furthermore, the theorem shows that the runtime can be highly
sensitive to the parameter settings. Notice that the margins between
the different runtime regimes are determined by the two parameters
$\epsilon$ and $\delta$ that can be set to any constants arbitrarily
close to 0. Hence, decreasing the selection pressure below
$\exp(\chi\sigma)$ by any constant, or increasing the mutation rate
above $\ln(\eta)/\sigma$ by any constant, will increase the 
runtime from polynomial to exponential.
Finally, note that the optimal mutation-selection balance
$\eta=\exp(\chi\sigma)$ depends on the problem parameter $\sigma$.
Hence, there exists no problem-independent optimal balance between the
selection pressure and the mutation rate.

Before proving Theorem~\ref{thm:main-result}, we mention that also
previous analyses have shown that the runtime of randomised search 
heuristics can depend critically on the parameter settings.  In 
the case of EAs, it is known that the population size is important
\cite{He2002,Jansen2005Offspring,Witt2006MuOneEA}. In fact, even small
changes to the population size can lead to an exponential increase in
the runtime \cite{Storch2008ParentPop,Witt2008Population}. Another
example is the evaporation factor in Ant Colony Optimisation, where a
small change can increase the runtime from polynomial to exponential
\cite{NeumannWitt2006Ant,Doerr2007ACO,Doerr2007Ant}. A distinguishing
aspect of the result in this paper is that the runtime is here 
shown to depend critically on the relationship between \emph{two} 
parameters of the algorithm.

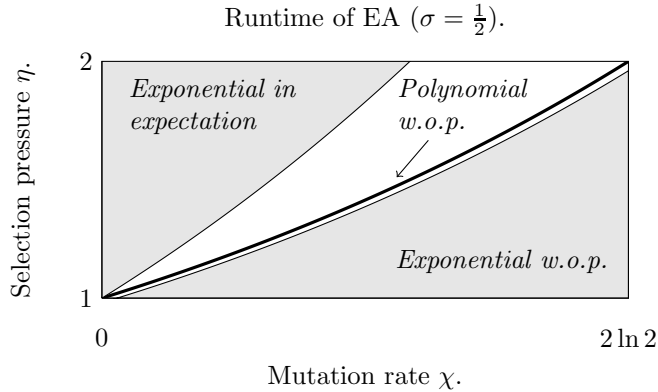
\begin{figure}
  \centering
  \begin{tikzpicture}[scale=0.7,domain=0:1.386294,samples=100,parametric=true]

    \begin{scope}
      \clip (0,0) rectangle (10,4.5);
      \draw[very thick]   plot[id=equal] function{(t/log(4))*10,4.50*(exp(t/2)-1)} node[midway,sloped,above] {};
      \draw[fill=gray!20] plot[id=upper] function{(t/log(4))*10,4.50*(2*exp(t/2)-2)}  -- (0,4.5) -- (0,0);
      \draw[fill=gray!20] plot[id=lower] function{(t/log(4))*10,4.50*(exp(t/2-0.02)-1)} -- (10,0.0) -- (0,0);
    \end{scope}

    \draw (0,0) -- (10,0) -- (10,4.5) -- (0,4.5) -- (0,0);
    \node              at (5,5.3)      {Runtime of EA ($\sigma=\frac{1}{2}$).};
    \node              at (5,-1.5)     {Mutation rate $\chi$.};
    \node              at (0, -0.75)   {0};
    \node              at (10,-0.75)   {$2\ln 2$};
    \node[rotate=90]   at (-1.5, 2.25) {Selection pressure $\eta$.};
    \node[anchor=east] at (0, 0)       {1};
    \node[anchor=east] at (0, 4.5)     {2};

    \draw[<-] (5.6,2.3) -- (6.2,3.0);
    \node[text width=2.5cm] at (2.3,3.65) {\emph{Exponential in expectation}};
    \node[text width=2.5cm] at (7.4,3.65) {\emph{Polynomial w.o.p.}};
    \node                   at (7.6,0.70) {\emph{Exponential w.o.p.}};

  \end{tikzpicture}
  \caption{Illustration of the main result (Theorem~\ref{thm:main-result}),
    indicating the runtime of the EA on \selpres for problem parameter $\sigma=1/2$, 
    as a function of the mutation rate $\chi$ (horizontal axis) 
    and the selection pressure $\eta$ (vertical axis).}
  \label{fig:finalresult}
\end{figure}

\section{Runtime Analysis}\label{sec:runtime-analysis}
This section gives the proofs of Theorem~\ref{thm:main-result}. The
analysis is conceptually divided into two parts. In Sections
\ref{sec:equilibrium} and \ref{sec:balance}, the behaviour of the main
``core'' of the population is analysed, showing that the population
enters an equilibrium state. This analysis is sufficient to prove the
polynomial upper bound in Theorem~\ref{thm:main-result}.
Sections~\ref{sec:embeddings} and \ref{sec:too-high} analyse the 
behaviour of the ``stray'' individuals that sometimes move away 
from the core of the population. This analysis is necessary to prove
the exponential lower bound in Theorem~\ref{thm:main-result}.

\subsection{Population Equilibrium}\label{sec:equilibrium}

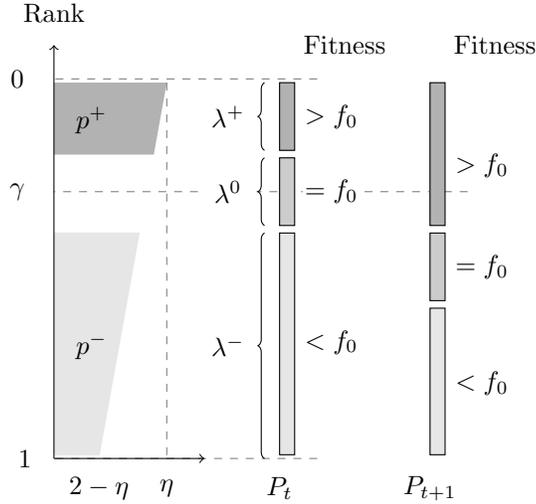
\begin{figure}
  \centering  
  \begin{tikzpicture}[bend angle=45,node distance=0.4cm,place]

    \draw[dashed,gray] (2.5,3.55) -- (-3,3.55) node [left=0.1cm,black] {$\gamma$};
    \draw[dashed,gray] (0.5,0.00) -- (-3,0.00) node [left=0.0cm,black] {$1$};
    \draw[dashed,gray] (0.5,5.05) -- (-3,5.05) node [left=0.1cm,black] {$0$};

    \filldraw[fill=gray!20] (0.00, 0.05) rectangle (0.20, 3.00);
    \filldraw[fill=gray!40] (0.00, 3.10) rectangle (0.20, 4.00);
    \filldraw[fill=gray!60] (0.00, 4.10) rectangle (0.20, 5.00);
    \node at (0,-0.4) {$P_t$};

    \draw[snake=brace] (-0.2,0.0) -- node[left=0.1cm]            {$\lambda^-$} (-0.2,3.0);
    \draw[snake=brace] (-0.2,3.1) -- node[left=0.1cm,fill=white] {$\lambda^0$} (-0.2,4.0);
    \draw[snake=brace] (-0.2,4.1) -- node[left=0.1cm]            {$\lambda^+$} (-0.2,5.0);

    \node[anchor=west] at (0.3,5.5) {Fitness};
    \node[anchor=west,fill=white] at (0.3,3.55) {$=f_0$};
    \node[anchor=west] at (0.3,4.55) {$>f_0$};
    \node[anchor=west] at (0.3,1.55) {$<f_0$};

    \filldraw[fill=gray!20] (2.00, 0.05) rectangle (2.20, 2.00);
    \filldraw[fill=gray!40] (2.00, 2.10) rectangle (2.20, 3.00);
    \filldraw[fill=gray!60] (2.00, 3.10) rectangle (2.20, 5.00);
    \node at (2,-0.4) {$P_{t+1}$};

    \node[anchor=west] at            (2.3,5.5) {Fitness};
    \node[anchor=west,fill=white] at (2.3,2.55) {$=f_0$};
    \node[anchor=west] at            (2.3,3.90) {$>f_0$};
    \node[anchor=west] at            (2.3,1.00) {$<f_0$};

    \begin{scope}
      \clip (-3,0) -- (-2.4,0) -- (-1.5,5) -- (-3,5) -- (-3,0);
      \filldraw[gray!20] (-3,0.05) rectangle (-1.5,3);
      \filldraw[gray!60] (-3,4.05) rectangle (-1.5,5);
      \node at (-2.5,4.5) {$p^+$};
      \node at (-2.5,1.5) {$p^-$};
    \end{scope}

    \draw (-2.4,-0.4) node {$2-\eta$};
    \draw (-1.5,-0.4) node {$\eta$};
    \draw[dashed,gray] (-1.5,5) -- (-1.5,0);

    \draw[->] (-3,0) --(-1,0);
    \draw[->] (-3,0) --(-3,5.5) node[above] {Rank};

  \end{tikzpicture}

  \caption{Impact of one generation of selection and mutation from the
    point of view of the $\gamma$-ranked individual in population
    $P_t$ \cite{Lehre2009FOGA}.}
  \label{fig:selection-bernoulli}
\end{figure}

As long as the global optimum has not been found, the population is
evolving with respect to the number of leading 1\nobreakdash-bits.  In
the following, we will prove that the population eventually reaches an
equilibrium state in which the population makes no progress with
respect to the number of leading 1\nobreakdash-bits.
The population equilibrium can be explained informally as follows.  On
one hand, the selection mechanism increases the number of individuals
in the population that have a relatively high number of leading
1\nobreakdash-bits. On the other hand, the mutation operator may flip
one of the leading 1\nobreakdash-bits, and the probability of doing so
clearly increases with the number of leading 1\nobreakdash-bits in the
individual. Hence, the selection mechanism causes an influx of
individuals with a high number of leading 1\nobreakdash-bits, and the
mutation causes an efflux of individuals with a high number of leading
1\nobreakdash-bits. At a certain point, the influx and efflux reach a
balance which is described in the field of population genetics as
mutation-selection balance.

Our first goal will be to describe the population when it is in the
equilibrium state. This is done rigorously by considering each
generation as a sequence of $\lambda$ Bernoulli trials, where each
trial consists of selecting an individual from the population and then
mutating that individual. Each trial has a certain probability of
being successful in a sense that will be described later, and the
progress of the population depends on the sum of successful trials,
\ie the population progress is a function of a certain Bernoulli
process.

\subsubsection{Ranking Selection as a Bernoulli Process}\label{sec:bernoulli}

We will associate a Bernoulli process with the selection step in any
given generation of the non-elitist EA, similar to Chen \etal
\cite{Chen2008Selection}.  For notational convenience, the individual
that has rank $\gamma\lambda$ in a given population, will be
called the $\gamma$-ranked individual of that population. For any
constant $\gamma, 0<\gamma<1$, assume that the $\gamma$-ranked
individual has $f_0:=\xi n$ leading \onebits for some
constant $\xi$. As illustrated in Fig.~\ref{fig:selection-bernoulli},
the population can be partitioned into three groups of individuals:
$\lambda^+$-individuals with fitness higher than $f_0$, 
$\lambda^0$-individuals with fitness equal to $f_0$, and 
$\lambda^-$-individuals with fitness less than $f_0$.  Clearly,
$\lambda^++\lambda^0+\lambda^-=\lambda$, and $0\leq
\lambda^+<\gamma\lambda$.

The following theorem makes a precise statement about the position
$\xi^*=\ln(\beta(\gamma)/\gamma)/\chi$ for a given rank
$\gamma,0<\gamma<1$, in which the population equilibrium
occurs. Informally, the theorem states that the number of leading
\onebits in the $\gamma$\nobreakdash-ranked individual is unlikely to decrease
when it is below $\xi^*n$, and is unlikely to increase, when it is
above $\xi^*n$.

\begin{theorem}\label{thm:eq-pos}
  For any constant $\gamma,0<\gamma<1$, and any $t_0>0$, define for
  all $t\geq 1$ the random variable $L_t$ as the number of leading 
  \onebits in the $\gamma$-ranked individual in 
  generation $t_0+t$.
  For any $t\leq e^{c\lambda}$, define 
  $T^*:=\min\{t,T-t_0\}$, where $T$ is the number of generations until an optimal search point
  is found. 
  Furthermore, for any constant mutation rate $\chi>0,$ define
  $\xi^* := \ln\left(\beta(\gamma)/\gamma\right)/\chi$, where 
  the function $\beta(\gamma)$ is as given in Eq. (\ref{eq:beta}).
  Then for any constant $\delta,0<\delta<\xi^*$, it holds that
  \begin{align*}
    & \prob{\min\left\{ \xi_0 n, (\xi^*-\delta)n \right\} 
          > \min_{0\leq i\leq T^*} L_i
          \mid L_0\geq \xi_0 n
         } 
    = e^{-\Omega(\lambda)}\\
    & \prob{\max\left\{ \xi_0 n, (\xi^*+\delta)n \right\} 
          < \max_{0\leq i\leq T^*}L_i
          \mid L_0\leq \xi_0 n
         }
     = e^{-\Omega(\lambda)}
  \end{align*}  
  where $c>0$ is some constant.
\end{theorem}
\begin{proof}
For \emph{the first statement}, define 
$\xi := \min\{ \xi_0 , \xi^* - \delta\}$.
Consider the events $\mathcal{F}^-_j$ and $\mathcal{G}^-_j$, defined
for $j,0\leq j<t,$ by
\[
  \mathcal{F}^-_j   : L_{j+1} < \xi n,\quad\text{and}\quad
  \mathcal{G}^-_j   : \min_{0\leq i\leq j} L_i\geq \xi n.
\]
The first probability in the theorem can now be expressed as
\begin{multline*}
  \prob{\cup_{0\leq j<T^*} \mathcal{F}_j^- \wedge  \mathcal{G}_j^- \mid L_0\geq \xi_0 n}\\
     \leq \sum_{j=0}^{t-1} \prob{\mathcal{F}_j^-\wedge \mathcal{G}_j^-\mid L_0\geq \xi_0 n}\\
    \leq \sum_{j=0}^{t-1} \prob{\mathcal{F}_j^-\mid \mathcal{G}_j^-\wedge L_0\geq \xi_0 n},
\end{multline*}
where the first inequality follows from the union bound. The second
inequality follows from the definition of conditional probability,
which is well-defined in this case because 
$\prob{\mathcal{G}_j^-\mid L_0\geq \xi_0 n}>0$ clearly holds.

To prove the first statement of the theorem, it 
now suffices to choose a not too large constant $c$, and show
that for all $j, 0\leq j<t$,
\begin{align*}
\prob{\mathcal{F}_j^-\mid \mathcal{G}_j^-\wedge L_0\geq \xi_0 n}=e^{-\Omega(\lambda)}.
\end{align*}

To show this, we consider each iteration of the selection mechanism
in generation $j$
as a Bernoulli trial, where a trial is successful if the following
event occurs:
\begin{itemize}
\item[$\mathcal{E}^+_1$:] An individual with at least $\xi n$
  leading 1\nobreakdash-bits is selected, and none of the initial $\xi n$ 
  bits are flipped.
\end{itemize}
Let the random variable $X$ denote the number of successful trials.
Notice that the event $X\geq \gamma\lambda$ implies that the 
$\gamma$-ranked individual in the next generation has at least 
$\xi n$ leading 1\nobreakdash-bits, \ie, that event $\mathcal{F}_j^-$
does not occur.
From the assumption that $\xi\leq \ln(\beta(\gamma)/\gamma)/\chi-\delta$,
we get
\begin{align*}
  \frac{1}{e^{\xi\chi}} \geq \frac{\gamma }{\beta(\gamma)}\cdot e^{\delta\chi}.
\end{align*}
Hence it follows that
\begin{align*}
\expect{X\mid \mathcal{G}_j^-\wedge L_0\geq \xi_0 n} 
   & = \lambda\cdot \prob{\mathcal{E}^+_1\mid \mathcal{G}_j^-\wedge L_0\geq \xi_0 n}\\
   & \geq \beta(\gamma)\lambda\cdot \left(1-\frac{\chi}{n}\right)\left(1-\frac{\chi}{n}\right)^{\xi n-1}\\
   & \geq \beta(\gamma)\lambda\cdot \left(1-\frac{\chi}{n}\right)\cdot e^{-\xi\chi}\\
   & \geq \gamma\lambda\cdot \left(1-\frac{\chi}{n}\right)\cdot e^{\delta\chi}\\
   & \geq \gamma\lambda\cdot (1+\delta\chi)\cdot \left(1-\frac{\chi}{n}\right).
\end{align*}
For sufficiently large $n$, a Chernoff bound
\cite{motwani:randomized} therefore implies that 
\begin{align*}
  \prob{X<\gamma\lambda\mid \mathcal{G}_j^-\wedge L_0\geq \xi_0 n}=e^{-\Omega(\lambda)}.  
\end{align*}

For \emph{the second statement}, define 
$\xi:=\max\{ \xi_0 , \xi^* + \delta \}$.
Consider the events $\mathcal{F}^+_j$ and $\mathcal{G}^+_j$, defined
for $j,0\leq j<t,$ by
\[
  \mathcal{F}^+_j   : L_{j+1} > \xi n,\quad\text{and}\quad
  \mathcal{G}^+_j   : \min_{0\leq i\leq j} L_i\leq \xi n.
\]
Similarly to above, the second statement
can be proved by showing that
\begin{align*}
  \prob{\mathcal{F}_j^+\mid\mathcal{G}_j^+\wedge L_0\leq \xi_0n} = e^{-\Omega(\lambda)}  
\end{align*}
for all $j,0\leq j<t.$
To show this, we define a trial in generation $j$
successful if one of the following two events occurs:
\begin{itemize}
\item[$\mathcal{E}^+_2$:] An individual with at least $\xi n+1$
  leading 1\nobreakdash-bits is selected, and none of the initial 
  $\xi n+1$ bits are flipped.
\item[$\mathcal{E}^-_2$:] An individual with less than $\xi n+1$
  leading 1\nobreakdash-bits is selected, and the mutation of this 
  individual creates an individual with at least $\xi n+1$ 
  leading 1\nobreakdash-bits.
\end{itemize}
Let the random variable $Y$ denote the number of successful trials.
Notice that the event $Y< \gamma\lambda$ implies that the 
$\gamma$\nobreakdash-ranked individual in the next generation has no
more than $\xi n$ leading 1\nobreakdash-bits, 
\ie, that event $\mathcal{F}_j^+$ does not occur.
Furthermore, since the $\gamma$\nobreakdash-ranked individual in the current
generation has no more than $\xi n$ leading 1-bits, less than 
$\gamma\lambda$ individuals have more than $\xi n$ 
leading 1\nobreakdash-bits. Hence, the event $\mathcal{E}^+_2$ occurs with probability
\begin{align*}
  \prob{\mathcal{E}^+_2\mid\mathcal{G}_j^+\wedge L_0\leq \xi_0n} 
    & \leq \beta(\gamma) \left(1-\frac{\chi}{n}\right)^{\xi n+1}
      \leq \frac{\beta(\gamma)}{e^{\xi\chi}}.
\end{align*}
If the selected individual has $k\geq 1$ 0-bits within the first $\xi
n+1$ bit positions, then the probability of mutating this individual
into an individual with at least $\xi n+1$ leading 1\nobreakdash-bits, and
hence also the probability of event $\mathcal{E}^-_2$, is bounded 
from above by 
\begin{align*}
  \prob{\mathcal{E}^-_2\mid\mathcal{G}_j^+\wedge L_0\leq \xi_0n} 
         & \leq \left(\frac{\chi}{n}\right)^k\left(1-\frac{\chi}{n}\right)^{\xi n+1-k}
           \leq \frac{\chi}{ne^{\xi\chi}}.
\end{align*}
From the assumption that 
$\xi\geq \ln(\beta(\gamma)/\gamma)/\chi+\delta$, we get
\begin{align*}
\frac{1}{e^{\xi\chi}} \leq \frac{\gamma}{\beta(\gamma)}\cdot e^{-\delta\chi}.
\end{align*}
Hence, for any constant $\delta',0<\delta'<1-e^{-\delta\chi}<1$, we have 
\begin{align*}
  \expect{Y\mid\mathcal{G}_j^+\wedge L_0\leq \xi_0n} 
     & =     \lambda\cdot \prob{\mathcal{E}_2^+\mid\mathcal{G}_j^+\wedge L_0\leq \xi_0n}\\
     &\quad +\lambda\cdot\prob{\mathcal{E}_2^-\mid\mathcal{G}_j^+\wedge L_0\leq \xi_0n}\\
     & \leq \lambda \left(\beta(\gamma)+\frac{\chi}{n}\right)\cdot e^{-\xi \chi}\\
     & \leq \gamma\lambda \left(1+\frac{\chi}{n\beta(\gamma)}\right)\cdot e^{-\delta\chi}\\
     & \leq \gamma\lambda (1-\delta') \left(1+\frac{\chi}{n\beta(\gamma)}\right).
\end{align*}
For sufficiently large $n$, a Chernoff bound therefore implies that 
\begin{align*}
\prob{Y\geq \gamma\lambda\mid\mathcal{G}_j^+\wedge L_0\leq \xi_0n}=e^{-\Omega(\lambda)}.
\end{align*}
\end{proof}

In the following, we will say that the $\gamma$-ranked individual $x$
is in the \emph{equilibrium position} with respect to a given constant
$\delta>0$, if the number of leading 1-bits in individual $x$ is larger than
$(\xi^*-\delta)n,$ and smaller than $(\xi^*+\delta)n$, where
$\xi^*=\ln(\beta(\gamma)/\gamma)/\chi$.

\subsubsection{Drift Analysis in Two Dimensions}\label{sec:drift}
Theorem~\ref{thm:eq-pos} states that when the population reaches a
certain region of the search space, the progress of the population will
halt and the EA enters an equilibrium state. Our next goal is to
calculate the expected time until the EA enters the equilibrium
state. More precisely, for any constants $\gamma, 0<\gamma<1$ and
$\delta>0$, we would like to bound the expected number of generations
until the fitness $f_0$ of the $\gamma$-ranked individual becomes at
least $(\ln(\beta(\gamma)/\gamma)/\chi-\delta)n$. Although the
fitness $f_0$ will have a tendency to drift towards higher values, it
is necessary to take into account that the fitness can in general both
decrease and increase according to stochastic fluctuations.

Drift analysis has proven to be a powerful mathematical technique to
analyse such stochastically fluctuating processes \cite{He2004}. Given
a distance measure (sometimes called potential function) from any
search point to the optimum, one estimates the drift $\Delta$
towards the optimum in one generation, and bounds the expected time 
to overcome a distance of $b(n)$ by $b(n)/\Delta$.

However, in our case, a direct application of drift analysis with
respect to $f_0$ will give poor bounds, because the drift of
$f_0$ depends on the value of a second variable $\lambda^+$. The
probability of increasing the fitness of the $\gamma$-ranked
individual is low when the number of individuals in the population
with higher fitness, \ie $\lambda^+$, is low. However, it is
still likely that the sum $\lambda^0+\lambda^+$ will increase, thus
increasing the number of good individuals in the population.

Several researchers have discussed this alternating behaviour of
population-based EAs \cite{Witt2006MuOneEA,Chen2008Selection}. Witt
shows that by taking into account replication of good individuals, one
can improve on trivial upper runtime bounds for the ($\mu$+$1$) EA,
e.g. from $O(\mu n^2)$ on \leadingones into $O(\mu n\log n + n^2)$
\cite{Witt2006MuOneEA}. Chen \etal describe a similar situation in
the case of an elitist EA, which goes through a sequence of
two-stage phases, where the first stage is characterised by
accumulation of leading individuals, and the second stage is
characterised by acquiring better individuals
\cite{Chen2008Selection}.

Generalised to the non-elitist EA described here, this corresponds
to first accumulation of $\lambda^+$-individuals, until one eventually
gains more than $\gamma\lambda$ individuals with fitness higher than
$f_0$. In the worst case, when $\lambda^+=0$, one expects that $f_0$
has a small positive drift. However, when $\lambda^+$ is high, there
is a high drift. When the fitness is increased, the value of
$\lambda^+$ is likely to decrease. To take into account this mutual
dependency between $\lambda^+$ and $f_0$, we apply drift analysis in
conceptually two dimensions, finding the drift of both $f_0$
and $\lambda^+$. Similar in vein to this two dimensional drift
analysis, is the analysis of simulated annealing due to Wegener, in
which a gambler's ruin argument is applied with respect to a potential
function having two components \cite{Wegener2005SA}.

The drift analysis applies the following simple property of function
$\beta$ which follows from its definition in
Eq.~(\ref{eq:beta}).
\begin{lemma}\label{lem:beta-property}
  For all $x\geq 1,$ and $\gamma,0<\gamma< 1$, the function $\beta$ defined in Eq. (\ref{eq:beta}) satisfies
  \begin{align*}
    \frac{\beta(\gamma/x)}{\beta(\gamma)}
     \geq \frac{1}{x},    
  \end{align*}
\end{lemma}

The following theorem shows that if the $\gamma$-ranked individual in
a given population is below the equilibrium position, then the
equilibrium position will be reached within expected $O(\lambda n^2)$
function evaluations.
\begin{theorem}\label{thm:eq-time-below}
  Let $\gamma$ and $\delta$ be any constants with $0<\gamma<1$ and
  $\delta>0$. The expected number of function evaluations until 
  the $\gamma$-ranked individual of the Linear Ranking EA 
  with population size $\lambda\geq c\ln n$, for some constant $c>0$
  that depends on $\gamma$, 
  attains at least $n(\ln(\beta(\gamma)/\gamma)/\chi-\delta)$ leading
  1\nobreakdash-bits or the optimum is reached, is $O(\lambda n^2)$.
\end{theorem}
\begin{proof}
  Recall from the definition of the EA that $P_t$ is the
  population vector in generation $t\geq 0$.  We consider the drift
  by to the potential function $h(P_t):=h_y(P_t)+\lambda
  h_x(P_t)$, 
  which is composed of a horizontal component $h_x$, and a
  vertical component $h_y$, defined as
  \begin{align*}
    h_x(P_t) & := n-\leadingones( x_{(\gamma)} ),\\
    h_y(P_t) & := \gamma\lambda-|\{ y\in P_t\mid f(y)>f(x_{(\gamma)}) \}|,
  \end{align*}
  where $x_{(\gamma)}$ is the $\gamma$-ranked individual in population
  $P_t$.  The horizontal $\Delta_{x,t}$ and vertical $\Delta_{y,t}$
  drift in generation $t$ are
  \begin{align*}
    \Delta_{x,t}(i) & := \expect{h_x(P_t)-h_x(P_{t+1})\mid h_x(P_t)=i}, \text{ and}\\
    \Delta_{y,t}(i) & := \expect{h_y(P_t)-h_y(P_{t+1})\mid h_y(P_t)=i}.
  \end{align*}
  The horizontal and vertical drift will be bounded
  independently in the following two cases, 
  \begin{itemize}
  \item[1)] $0\leq \lambda^+_t\leq \gamma\lambda/l$,\quad and
  \item[2)] $\gamma\lambda/l<\lambda^+_t$,
  \end{itemize}
  where $l$ is a constant that will be specified later,

  Assume that the $\gamma$-ranked individual has $\xi n$ leading
  1\nobreakdash-bits, where it holds
  ${\xi < \ln(\beta(\gamma)/\gamma)/\chi-\delta}$.
  By the first statement of Theorem~\ref{thm:eq-pos}, the probability
  of reducing the number of leading \onebits in the $\gamma$-ranked
  individual, i.e., of increasing the horizontal distance, is 
  $e^{-\Omega(\lambda)}$. The horizontal distance cannot increase 
  by more than $n$, so $\Delta_{x,t}\geq -ne^{-\Omega(\lambda)}$ holds
  in both cases.

  We now bound the horizontal drift $\Delta_{x,t}$ for Case 2.
  Let the random variable $S_t$ be the number of selection 
  steps in which an individual with fitness strictly higher 
  than $f_0=f(x_{(\gamma)})$ is selected, and none of the 
  leading $\xi n$ bits are flipped. Then
  \begin{align*}
    \expect{S_t} & \geq \lambda \cdot \beta(\gamma/l)\cdot e^{-\xi\chi}\cdot\left(1-\frac{\chi}{n}\right)\\
    & \geq \gamma\lambda \cdot (1+\chi\delta)\cdot\frac{\beta(\gamma/l)}{\beta(\gamma)}\cdot\left(1-\frac{\chi}{n}\right)\\
    & \geq \gamma\lambda \cdot \frac{(1+\chi\delta)}{l}\cdot\left(1-\frac{\chi}{n}\right).
  \end{align*}
  By defining $l:=(1+\chi\delta/2)$, there exists a constant $\delta'>0$
  such that for sufficiently large $n$, we have $\expect{S_t}\geq
  (1+\delta')\cdot \gamma\lambda.$ Hence, by a Chernoff bound, with
  probability $1-e^{-\Omega(\lambda)}$, the number $S_t$ of such
  selection steps is at least $\gamma\lambda$, in which case
  $\Delta_{x,t}\geq 1$. The unconditional horizontal drift in Case 2
  therefore satisfies
  $\Delta_{x,t} \geq 1\cdot (1-e^{-\Omega(\lambda)})-n\cdot e^{-\Omega(\lambda)}$.

  We now bound the vertical drift $\Delta_{y,t}$ for Case 1.  In order to
  generate a $\lambda^+$-individual in a selection step, it is
  sufficient that a $\lambda^+$-individual is selected and none of the
  leading ${\xi n+1}$ 1\nobreakdash-bits are flipped. We first show
  that the expected number of such events is sufficient to ensure a
  non-negative drift. If $\lambda_t^+=0$, then the vertical drift
  cannot be negative. Let us therefore assume that
  $0<\lambda_t^+=\gamma\lambda/m$ for some $m>1$ which is not necessarily
  constant. The expected number of times a new $\lambda^+$-individual
  is created is at least
  \begin{align*}
    \lambda\cdot\beta(\gamma/m)\cdot e^{-\xi\chi} \cdot\left(1-\frac{\chi}{n}\right)
    &\geq \gamma\lambda\cdot\frac{\beta(\gamma/m)}{\beta(\gamma)}\cdot(1+\chi\delta)\cdot\left(1-\frac{\chi}{n}\right)\\
    &\geq (\lambda \gamma/m) \cdot(1+\chi\delta)\cdot\left(1-\frac{\chi}{n}\right).
  \end{align*}
  Hence, for sufficiently large $n$, this is at least $\lambda_t^+$, and
  the expected drift is at least positive.
  In addition, a $\lambda^+$-individual can be created by selecting a
  $\lambda^0$-individual, and flipping the first 0-bit and no other
  bits. The expected number of such events is at least
  $\lambda\cdot\beta(\gamma/l,\gamma)\cdot e^{-\xi\chi}\cdot \chi/n=\Omega(\lambda/n)$.
  Hence, the expected vertical drift in Case 1 is
  $\Omega(\lambda/n)$. Finally, for Case 2, we use the trivial lower
  bound $\Delta_{y,t}\geq -\gamma\lambda$.

  The horizontal and vertical drift is now added into a \emph{combined drift}
  \begin{align*}
    \Delta_t := \Delta_{y,t} + \lambda\Delta_{x,t},     
  \end{align*}
  which in the two cases is bounded by
  \begin{itemize}
  \item[1)] $\Delta_t = \Omega(\lambda/n)-\lambda n e^{-\Omega(\lambda)}$,\quad and
  \item[2)] $\Delta_t =
    -\gamma\lambda+\lambda(1-e^{-\Omega(\lambda)})-\lambda n e^{-\Omega(\lambda)}$.
  \end{itemize}

  Given a population size $\lambda\geq c\ln n$, for a sufficiently
  large constant $c$ with respect to $\gamma$, the combined drift
  $\Delta_t$ is therefore in both cases bounded from below by
  $\Omega(\lambda/n).$ The maximal distance is $b(n)\leq
  (n+\gamma)\cdot\lambda$, hence, the expected number of function
  evaluations $T$ until the $\gamma$-ranked individual attains at
  least $n(\ln(\beta(\gamma)/\gamma)/\chi-\delta)$ leading \onebits is
  no more than $\expect{T}\leq \lambda\cdot b(n)/\Delta_t = O(\lambda
  n^2)$.
\end{proof}

\subsection{Mutation-Selection Balance}\label{sec:balance}
In the previous section, it was shown that the population reaches an
equilibrium state in $O(\lambda n^2)$ function evaluations in
expectation. Furthermore, the position of the equilibrium state is
given by the selection pressure $\eta$ and the mutation rate $\chi$.
By choosing appropriate values for the parameters $\eta$ and $\chi$,
one can ensure that the equilibrium position occurs close to the
global optimum that is given by the problem parameter $\sigma$.
Theorem \ref{thm:unlikely-above}, that will be proved in
Section~\ref{sec:too-low}, also implies that no individual will reach
far beyond the equilibrium position. It is now straightforward to
prove that an optimal solution will be found in polynomial
time with overwhelmingly high probability.

\begin{theorem}\label{thm:succprob-corr-selpres}
  The probability that Linear Ranking EA with population size
  $n\leq\lambda\leq n^k$, for any constant integer $k\geq 1$, selection
  pressure $\eta$, and bit-wise mutation rate $\chi/n$ for a constant
  $\chi>0$ satisfying $\eta=\exp(\sigma\chi)$, finds the optimum of 
  \selpres within $n^{k+4}$ function evaluations is $1-e^{-\Omega(n)}$.
\end{theorem}
\begin{proof}
  We divide the run into two phases. The first phase lasts the first
  $\lambda n^3$ function evaluations, and the second phase lasts the
  remaining $n^{k+4}-\lambda n^3$ function evaluations. We say that a
  failure occurs during the run, if within these two phases, there
  exists an individual that has more than $(\sigma+\delta)n$ leading
  \onebits, or more than $2n\delta/3$ \onebits in the interval from
  $(\sigma+\delta)n$ to $(\sigma+2\delta)n$. We first claim that the
  probability of this failure event is exponentially small. By
  Theorem~\ref{thm:unlikely-above}, no individual reaches more than
  $(\sigma+\delta)n$ leading \onebits within $cn^{k+4}$ function
  evaluations with probability $1-e^{-\Omega(n)}$.  Hence the bits
  after position $(\sigma+\delta)n$ will be uniformly distributed. By
  a Chernoff bound, and a union bound over all the individuals in the
  two phases, the probability that any individual during the two
  phases has more than $2\delta n/3$ \onebits in the interval from
  $n(\sigma+\delta)$ to $n(\sigma+2\delta)$ is exponentially small. We
  have therefore proved the first claim.

  Let $\gamma>0$ be a constant such that
  $\ln(\beta(\gamma)/\gamma)/\chi>\sigma-\delta$.  We say that a
  failure occurs in the first phase, if by the end of this phase,
  there exists a non-optimal individual with rank between 0 and
  $\gamma$ that has less than $(\sigma-\delta)n$ leading \onebits.  We
  will prove the claim that the probability of this failure event is
  exponentially small.  By Theorem~\ref{thm:eq-time-below}, the
  expected number of function evaluations until the $\gamma$-ranked
  individual has obtained at least $(\sigma-\delta)n$ leading \onebits
  is no more than $c\lambda n^2$, for some constant $c>0$. We divide
  the first phase into sub-phases, each of length $2c\lambda n^2$. By
  Markov's inequality, the probability that the
  $\gamma$\nobreakdash-ranked individual has not obtained
  $(\sigma-\delta)n$ leading \onebits within a given sub-phase is less
  than $1/2$. The probability that this number of leading \onebits is
  not achieved within $n/2c$ such sub-phases, \ie by the end of the
  first phase, is no more than $2^{-n/2c}$, and the second claim
  holds.

  We say that a failure occurs in the second phase, if a non-optimal
  individual with rank better than $\gamma$ has less than
  ${(\sigma-\delta)n}$ leading \onebits, or the optimum is not found by
  the end of the phase. We claim that the probability of this failure
  event is exponentially small. The first part of the claim follows
  from the first part of Theorem \ref{thm:eq-pos} with the parameters
  $\xi_0=\sigma-\delta$ and $t=n^{k+4}/\lambda-n^3$.  Assuming no
  failure in the previous phase, it suffices to select an individual
  with rank between $0$ and $\gamma$, and flip the leading $k+3$
  \onebits, and no other bits. The probability that this event happens
  during a single selection step, assuming that $n>2\chi-k-3$, \ie,
  $n-k-3<2n-2\chi$, is
  \begin{align*}
    r 
    & =    \beta(\gamma)\left(\frac{\chi}{n}\right)^{k+3}\left(1-\frac{\chi}{n}\right)^{n-k-3}\\
    &  \geq \beta(\gamma)\left(\frac{\chi}{n}\right)^{k+3}\left[\left(1-\frac{\chi}{n}\right)^{\frac{n}{\chi}-1}\right]^{2\chi}\\
    &  \geq \frac{\beta(\gamma)}{e^{2\chi}}\left(\frac{\chi}{n}\right)^{k+3}.
  \end{align*}
  The expected number of selection steps until the optimum is produced
  is $1/r\leq c'n^{k+3}$ for some constant $c'>0$. Similarly to the
  first phase, we consider sub-phases, each of length $2c'n^{k+3}$. By
  Markov's inequality, the probability that the optimum has not been
  found within a given sub-phase is less than $1/2$. The probability
  that the optimum has not been found within $n/4c'$ sub-phases, \ie
  before the end of the second phase, is $2^{-n/4c'}$, and the third
  claim holds.

  If none of the failure events occurs, then the optimum has been
  found by the end of the second phase.  The probability that any of
  the failure events occurs is $e^{-\Omega(n)}$, and the theorem then
  follows.
\end{proof}

\subsection{Non-Selective Family Trees}\label{sec:embeddings}

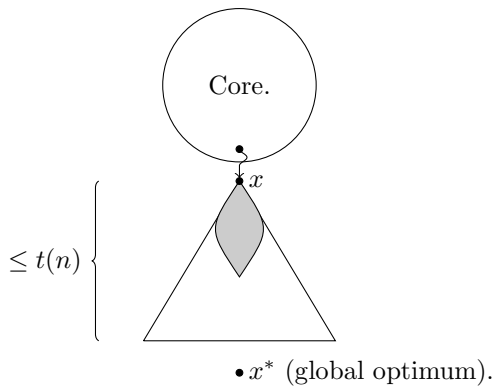
\begin{figure}
  \centering  
  \begin{tikzpicture}[level distance=4mm,sibling distance=4mm,scale=0.85]
    \draw (0.5,1) -- (3.5,1) -- (2,3.5) -- (0.5,1);

    \draw[fill=black!20] 
    (2,3.5) .. controls (1.5,2.75) .. (2,2.0) ..controls (2.5,2.75) .. (2,3.5);

    \draw[snake=brace] (-0.2,1.0) -- node[left=0.1cm]  {$\leq t(n)$} (-0.2,3.5);

    \draw (2,5) circle (1.2cm) node {Core.};
    \draw[fill] (2,4.0) circle (1.5pt); 
    \draw[fill] (2,3.5) circle (1.5pt) node[anchor=west] {$x$};
    \draw[->,
    snake=snake,
    segment aspect=0,
    segment amplitude=1mm,
    line after snake=0.5mm] (2,4.0) -- (2,3.55);

    \draw[fill] (2,0.5) circle (1.5pt) node[anchor=west] { $x^*$ (global optimum).};
  \end{tikzpicture}
  \caption{Non-selective family tree (triangle) of the 
           family tree (gray) rooted in individual $x$ \cite{Lehre2009FOGA}.}
  \label{fig:core}
\end{figure}

While Theorem \ref{thm:eq-pos} describes the equilibrium position of
any $\gamma$-ranked individual for any positive constant $\gamma$, the
theorem cannot be used to analyse the behaviour of single ``stray''
individuals, including the position of the fittest individual (\ie
$\gamma=0$). This is because the tail inequalities obtained by the
Chernoff bounds used in the proof of Theorem \ref{thm:eq-pos} are too
weak for ranks of order $\gamma=o(1)$.

To analyse stray individuals, we will apply the technique of
non-selective family trees introduced in \cite{Lehre2009FOGA}. This
technique is different from, but related to, the \emph{family tree}
technique described by Witt \cite{Witt2006MuOneEA}. A family tree has
as its root a given individual $x$ in some generation $t$, and the
nodes in each level $k$ correspond to the subset of the population in
generation $t+k$ defined in the following way. An individual $y$ in
generation $t+k$ is a member of the family tree if and only if it was
generated by selection and mutation of an individual $z$ that belongs
to level $t+k-1$ of the family tree. In this case, individual $z$ is
the parent node of individual $y$. If there is a path from an
individual $z$ at level $k$ to an individual $y$ at level $k'>k$, then
individual $y$ is said to be a \emph{descendant} of individual $z$,
and individual $z$ is an \emph{ancestor} of individual $y$. A directed
path in the family tree is called a \emph{lineage}. A family tree is
said to become \emph{extinct} in generation $t+t(n)+1$ if none of the
individuals in level $t(n)$ of the tree were selected. In this case,
$t(n)$ is called the \emph{extinction time} of the family tree.

The idea for proving that stray individuals do not reach a given part
of the search space can be described informally using Fig.
\ref{fig:core}. One defines a certain subset of the search space
called the \emph{core} within which the majority of the population is
confined with overwhelming probability. In our case, an appropriate
core can be defined using Theorems \ref{thm:eq-pos} and
\ref{thm:eq-time-below}. One then focuses on the family trees that are
outside this core, but which have roots within the core. Note that
some descendants of the root may re-enter the core. We therefore prune
the family tree to those descendants which are always outside the
core. More formally, the pruned family tree contains node $x$ if and
only if $x$ belongs to the original family tree, and $x$ and all
its ancestors are outside the core.

We would then like to analyse the positions of the individuals that
belong to the pruned family tree.  However, it is non-trivial to
calculate the exact shape of this family tree. Let the random variable
$\xi_x$ denote the number of offspring of individual $x$. Clearly, the
distribution of $\xi_x$ depends on how $x$ is ranked within the
population. Hence, different parts of the pruned family tree may grow
at different rates, which can influence the position and shape of the
family tree. To simplify the analysis, we embed the pruned family tree
into a larger family tree which we call the \emph{non-selective
  family tree}.  This family tree has the same root as the real pruned
family tree, however it grows through a modified selection process. In
the real pruned family tree, the individuals have different numbers of
offspring according to their rank in the population.  In the
non-selective family tree, the offspring distribution $\xi_x$ of all
individuals $x$ is identical to the offspring distribution $\xi_z$ of an
individual $z$ which is best ranked among individuals outside the
core.  We will call the expectation of this distribution $\xi_z$ the 
\emph{reproductive rate} of the non-selective family tree.
Hence, each individual in the non-selective family tree has at
least as many offspring as in the real family tree. The real family tree will therefore occur as a sub-tree in the non-selective
family tree. Furthermore, the probability that the real family tree
reaches a given part of the search space is upper bounded by the
probability that the non-selective family tree reaches this part of the
search space. A related approach, where faster growing family trees
are analysed, is described by J\"agersk\"upper and 
Witt~\cite{Jaegerskuepper2005MuPlusOne}.

Approximating the family tree by the non-selective family tree has three
important consequences.  The \emph{first} consequence is that the
non-selective family tree can grow faster than the real family tree, and
in general beyond the population size $\lambda$ of the original
process.  The \emph{second} consequence is that since all individuals
in the family tree have the same offspring distribution, no individual
in the family tree has any selective advantage, hence the name
non-selective family tree. The behaviour of the 
family tree is therefore independent of the fitness function, and each lineage fluctuates
randomly in the search space according to the bits flipped by the
mutation operator. Such mutation random walks are easier to analyse
than the real search process. To bound the probability that such a
mutation random walk enters a certain region of the search space, it is
necessary to bound the extinction time $t(n)$ of the non-selective
family tree.  The \emph{third} consequence is that the sequence of
random variables $Z_{t\geq 0}$ describing the number of elements in
level $t$ of the non-selective family tree is a discrete time branching
process \cite{Haccou2005BranchingProcesses}. We can therefore apply
the techniques that have been developed to study branching processes 
to bound the extinction time $t(n)$.

Before introducing branching processes, we summarise the main steps in
a typical application of non-selective family trees, assuming the goal
is to prove that with overwhelming probability, an algorithm does not
reach a given search point $x^*$ within $e^{cn}$ generations for some
constant $c>0$.  The first step is to define an appropriate core,
which is a subset of the search space that is separated from $x^*$ by
some distance. The second step is to prove that any non-selective
family tree outside the core will become extinct in $t(n)$ generations
with overwhelmingly high probability. This can be proved by applying
results about branching processes,
e.g. Lemma~\ref{lemma:branching-width} and
Lemma~\ref{lemma:branching-size-tail} in this paper.  The third step
is to bound the number of different lineages that the family tree has
within $t(n)$ generations. Again, results about branching processes
can be applied. The fourth step involves bounding the probability that
a given lineage, starting inside the core reaches search point $x^*$
within $t(n)$ generations. This can be shown in various ways,
depending on the application. The fifth, and final step, is to apply a
union bound over all the different lineages that can exist within
$e^{cn}$ generations.

In the second step, one should keep in mind that there are several
causes of extinction. A reproductive rate less than $1$ is perhaps the
most evident cause of extinction. Such a low reproductive rate may
occur when the fitness outside the core is lower than the fitness
inside the core, as is the case for the family trees considered in
Section~\ref{sec:too-high}. With a majority of the population inside
the core, each individual outside the core is selected in expectation
less than once per generation. However, a low reproductive rate is not
the only cause of extinction. This is illustrated by the core
definition in Section~\ref{sec:too-low}, where the fitness is
generally higher outside, than inside the core. While the family tree
members may in general be selected more than once per generation, the
critical factor here is that their offspring are in expectation closer
to the core than their parents. Hence, the lineages outside the core
will have a tendency to drift back into the core where they are no
longer considered part of the family tree due to the pruning process.

\begin{definition}[Single-Type Branching Process \cite{Haccou2005BranchingProcesses}]
  A single-type branching process is a Markov process $Z_0, Z_1, ...$
  on $\mathbb{N}_0$, which
  for all $t\geq 0$, is given by $Z_{t+1}  := \sum_{i=1}^{Z_t} \xi_i$,
  where $\xi_i\in \mathbb{N}_0$ are i.i.d. random variables having $\expect{\xi}=:\rho$.
\end{definition}
A branching process can be thought of as a population of identical
individuals, where each individual survives exactly one
generation. Each individual produces $\xi$ offspring independently of
the rest of the population during its lifetime, where $\xi$ is a
random variable with expectation $\rho$. The random variable 
$Z_{t}$ denotes the population size in generation $t$. 
Clearly, if $Z_{t}=0$ for some $t$, then $Z_{t'}=0$ for all 
$t'\geq t$. The following lemma gives a simple bound on the size
of the population after $t\geq 1$ generations.
\begin{lemma}\label{lemma:branching-width}
  Let $Z_0, Z_1, ...$ be a single-type branching process with
  $Z_0:=1$ and mean number of offspring per individual $\rho$.
  Define random variables $T:=\min\{t\geq 0\mid Z_t=0\}$, \ie the 
  extinction time, and $X_t$ the number of different lineages 
  until generation $t$. Then for any $t,k\geq 1$,
  \begin{align*}
    \prob{Z_t\geq k}  \leq \frac{\rho^t}{k},\quad\text{ and }\quad
    \prob{T\geq t}   \leq \rho^t.
  \end{align*}
  Furthermore, if $\rho<1$, then
  \begin{align*}
    \expect{X_t} \leq \frac{\rho}{1-\rho}, \quad\text{ and }\quad
    \prob{X_t\geq k} \leq \frac{\rho}{k(1-\rho)}.
  \end{align*}
\end{lemma}
\begin{proof}
  By the law of total expectation, we have
  \begin{align*}    
  \expect{Z_t}
      =    \expect{\expect{Z_t\mid Z_{t-1}}}
      =    \rho\cdot \expect{Z_{t-1}}.
  \end{align*}  
  Repeating this $t$ times gives
  $\expect{Z_t} =  \rho^t\cdot \expect{Z_0}$.
  The first part of the lemma now follows by Markov's
  inequality, \ie 
  \begin{align*}
    \prob{Z_t\geq k} \leq \frac{\expect{Z_t}}{k} = \frac{\rho^t}{k}.    
  \end{align*}
  The second part of the lemma is a special case of the first part 
  for $k=1$, \ie $\prob{T\geq t}=\prob{Z_t\geq 1}\leq \rho^t$.
  For the last two parts, note that since each lineage must contain 
  at least one individual
  that is unique to that lineage, we have $X_t\leq Z_1+\cdots+Z_t$. 
  By linearity of expectation and the previous 
  inequalities, we can therefore conclude that 
  \begin{align*}
     \expect{X_t}
     \leq \sum_{i=1}^t \expect{Z_i}
     \leq \sum_{i=1}^\infty \rho^i
     = \frac{\rho}{1-\rho}.
  \end{align*}
  Finally, it follows from Markov's inequality that 
  \begin{align*}
    \prob{X_t\geq k}\leq \frac{\rho}{k(1-\rho)}.    
  \end{align*}
\end{proof}

From the preceding lemma, it is clear that the expected number of
offspring $\rho$ is important for the fate of a branching process. For
$\rho<1$, the process is called \emph{sub-critical}, for ${\rho=1}$, the
process is called \emph{critical}, and for $\rho>1$, the process is
called \emph{super-critical}. In this paper, we will consider
sub-critical processes.

\subsection{Too High Selection Pressure}\label{sec:too-high}

In this section, it is proved that \selpres is hard for Linear Ranking
EA when the ratio between parameters $\eta$ and $\chi$ is sufficiently
large. The overall proof idea is first to show that the population is
likely to reach the equilibrium position before the optimum is reached
(Proposition~\ref{prop:prob-opt-in-n-square} and
Theorem~\ref{thm:eq-time-below}).  Once the equilibrium position is
reached, a majority of the population will have significantly more
than $(\sigma+\delta) n$ leading \onebits, and individuals that are
close to the optimum are therefore less likely to be selected
(Proposition \ref{prop:belowcritical}).

The proof of Proposition~\ref{prop:prob-opt-in-n-square} builds on the
result in Proposition~\ref{prop:111-ind}, which states that the
individuals with at least $k+3$ leading \onebits will quickly dominate
the population. Hence, family trees of individuals with less than
$k+3$ leading \onebits are likely to become extinct before they
discover an optimal search point. Recall that optimal search points
have $k+3$ leading \zerobits.  In the following, individuals with at
least $k+3$ leading \onebits will be called $1^{k+3}$-individuals.

\begin{proposition}\label{prop:111-ind}
  Let $\gamma^*$ be any constant $0<\gamma^*<1$, and $t(\lambda)=\poly(\lambda)$.
  If the Linear Ranking EA with population size $\lambda, n\leq
  \lambda\leq n^k$, for any constant integer $k\geq 1$, 
  and bit-wise mutation rate $\chi/n$ for a constant $\chi>0$,
  is applied to
  \selpres, then with probability $1-o(1)$, all the
  $\gamma^*$\nobreakdash-ranked individuals in generation $\log\lambda$
  to generation $T^*:=\min\{t(\lambda),T-1\}$ are $1^{k+3}$-individuals, 
  where $T$ is the number of generations until the optimum has been found.
\end{proposition}  
\begin{proof}
  If the $\gamma^*$-ranked individual in some generation $t_0\leq \log\lambda$ 
  is an $1^{k+3}$\nobreakdash-individual, then by the first part of 
  Theorem~\ref{thm:eq-pos} with parameter $\xi_0:=(k+3)/n$,
  the $\gamma^*$-ranked individual remains so until generation
  $T^*$ with probability $1-e^{-\Omega(\lambda)}$. Otherwise,
  we consider the run a failure.

  It remains to prove that the $\gamma^*$-ranked individual in one of
  the first $\log\lambda$ generations is an $1^{k+3}$-individual with
  probability $1-o(1)$.  We apply the drift theorem with respect to
  the potential function $\log(\lambda^+)$, where $\lambda^+$ is the
  number of $1^{k+3}$-individuals in the population.

  A run is considered failed if the fraction of
  $1^{k+3}$\nobreakdash-individuals in any of the first $T^*$
  generations is less than $\gamma_0:=1/2^{k+4}.$ The initial
  generation is sampled uniformly at random, so by a Chernoff bound,
  the probability that the fraction of $1^{k+3}$-individuals in the
  initial generation is less than $\gamma_0$, is
  $e^{-\Omega(\lambda)}$.  Given that the initial fraction of
  $1^{k+3}$-individuals is at least $\gamma_0$, it follows again by the
  first part of 
  Theorem~\ref{thm:eq-pos} with parameter $\xi_0=(k+3)/n$ that this 
  holds until generation $T^*$ with probability $1-e^{-\Omega(\lambda)}$.  Hence, the
  probability of this failure event is $e^{-\Omega(\lambda)}$.

  The $1^{k+3}$\nobreakdash-individuals are fitter than any other
  non-optimal individuals.  Assume that the fraction of
  $1^{k+3}$-individuals in a given generation is $\gamma, \gamma_0\leq
  \gamma<\gamma^*$. In order to create a $1^{k+3}$-individual in a
  selection step, it suffices to select one of the best
  $\gamma\lambda$ individuals, and to not mutate any of the first $k+3$
  bit positions. The expected number of $1^{k+3}$-individuals in the
  following generation is therefore at least $r(\gamma)\lambda$, where
  we define
  $r(\gamma) := \beta(\gamma)(1-\chi/n)^{k+3}$.  The ratio
  $r(\gamma)/\gamma$ is linearly decreasing in $\gamma$, and for
  sufficiently large $n$, strictly
  larger than $1+c$, where $c>0$ is a constant. Hence, for
  all $\gamma<\gamma^*$, it holds that
  \begin{align*}
    r(\gamma)=\gamma\frac{r(\gamma)}{\gamma}> \gamma\frac{r(\gamma^*)}{\gamma^*}\geq \gamma(1+c).
  \end{align*}
  The drift is therefore for all $\gamma,$ where $\gamma_0\leq \gamma<\gamma^*$,
  \begin{align*}
    \Delta 
    & \geq \log(r(\gamma)\lambda)-\log(\gamma\lambda)\\
    & \geq \log(\gamma(1+c)\lambda)-\log(\gamma\lambda)
     =\log(1+c).    
  \end{align*}

  Assuming no failure, the potential must be increased by no more than
  $b(\lambda):=\log(\gamma^*\lambda)-\log(\gamma_0\lambda)=\log(\gamma^*/\gamma_0)$.
  By the drift theorem, the expected number of generations until this
  occurs is $b(\lambda)/\Delta=O(1)$. And the probability that this
  does not occur within $\log\lambda$ generations is
  $O(1/\log\lambda)$ by Markov's inequality.

  Taking into account all the failure probabilities, the proposition
  now follows.
\end{proof}

\begin{proposition}\label{prop:prob-opt-in-n-square}
  For any constant $r>0$, the probability that
  the Linear Ranking EA with population size 
  $\lambda, n\leq\lambda\leq n^k$, for some constant integer $k\geq
  1$, and bit-wise mutation rate $\chi/n$ for a constant $\chi>0$,
  has not found the optimum of \selpres within
  $\lambda rn^2$ function evaluations is $\Omega(1)$.
\end{proposition}
\begin{proof}
  We consider the run a failure if at some point between generation
  $\log\lambda$ and generation $rn^2$, the $(1+\delta)/2$-ranked
  individual has less than $k+3$ leading \onebits without first
  finding the optimum. By Proposition \ref{prop:111-ind}, the
  probability of this failure event is $o(1)$.

  Assuming that this failure event does not occur, we apply the method
  of non-selective family trees with the set of $1^{k+3}$-individuals
  as core. Recall that the family trees are pruned such that they only
  contain lineages outside the core.  However, to simplify the
  analysis, the family trees will not be pruned before generation
  $\log\lambda$.  Therefore, any family tree that is not rooted in an
  $1^{k+3}$-individual, must be rooted in the initial population. The
  proof now considers the family trees with roots after and before
  generation $\log\lambda$ separately.

  \emph{Case 1:} We firstly consider the at most $m:=\lambda rn^2\leq
  rn^{k+2}$ family trees with roots after generation $\log\lambda$. We
  begin by estimating the total number of lineages, and their
  extinction times. The mean number of offspring $\rho$, of an
  individual with rank $\gamma$, is no more than $\alpha(\gamma)$, as
  given in Eq.~(\ref{eq:alpha}).  Assuming no failure, any non-optimal
  individual outside the core has rank at least
  $\gamma:=(1+\delta)/2$.  Hence for any selection pressure
  $\eta,1<\eta\leq 2$, the mean number of offspring of an individual
  in the family tree is $\rho\leq
  \alpha((1+\delta)/2)=1-(\eta-1)\delta<1$.
  We consider the run a failure if any of the $m$ family trees
  survives longer than $t:=(k+3)\ln n/\ln(1/\rho)$ generations.
  By the union bound and Lemma~\ref{lemma:branching-width}, the
  probability of this failure event is no more than
  $m\rho^t=mn^{-k-3}=O(1/n).$ 

  Let the random variable $P_i$ be the number of lineages in family
  tree $i,1\leq i\leq m$. The expected number of lineages in a given
  family tree is by Lemma~\ref{lemma:branching-width} no more than
  $\rho/(1-\rho)$.  We consider the run a failure if there are more
  than $2m\rho/(1-\rho)$ lineages in all these family trees.  The
  probability of this failure event is by Markov's inequality no more
  than
  \begin{align*}
    \prob{\sum_{i=1}^m P_i\geq \frac{2m\rho}{1-\rho}}
    \leq \frac{(1-\rho)\sum_{i=1}^m \expect{P_i}}{2m\rho}
    \leq 1/2.
  \end{align*}

  We now bound the probability that any given lineage contains a
  $0^{k+3}$-individual, which is necessary to find an optimal search
  point. The probability of flipping a given bit during $t$
  generations is by the union bound no more than $t\chi/n$, and the
  probability of flipping $k+3$ bits within $t$ generations is no more
  than $(t\chi/n)^{k+3}$. The probability that any of the at most
  $2m\rho/(1-\rho)$ lineages contains a $0^{k+3}$-individual is by
  the union bound no more than 
  \begin{align*}
    \frac{(t\chi/n)^{k+3}2m\rho}{1-\rho}=O(\ln^{k+3} n/n).    
  \end{align*}

  \emph{Case 2:} We secondly consider the family trees with roots
  before generation $\log\lambda$. In the analysis, we will not prune
  these family trees during the first $\log\lambda$
  generations. However, after generation $\log\lambda$, the family
  trees will be pruned as usual. This will only overestimate the
  extinction time of the family trees. Furthermore, there will be
  exactly $\lambda$ such family trees, one family tree for each of the
  $\lambda$ randomly chosen individuals in the initial population.

  We now bound the number of lineages in these family trees, and their
  extinction times. The mean number of offspring is no more than
  $\eta\leq 2$ during the first $\log\lambda$ generations.  Because
  the family trees are pruned after generation $\log\lambda$, we can
  re-use the arguments from case 1 above to show that the mean number
  of offspring after generation $\log\lambda$ is no more than $\rho$,
  for some constant $\rho<1$.  Let random variable $Z_t$ be the number
  of family tree members in generation $Z_t$. Analogously to the proof
  of Lemma~\ref{lemma:branching-width}, we have $\expect{Z_t} \leq
  2^t$ if $t\leq\log\lambda$, and $\expect{Z_t} \leq
  2^{\log \lambda} \rho^{t-\log\lambda} =
  \lambda\rho^{t-\log\lambda}$ for $t\geq\log\lambda$.  We consider
  the run a failure if any of the $\lambda$ family trees survives
  longer than $\sqrt{n}$ generations.  By the union bound and Markov's
  inequality, the probability of this failure event is no more than
  $\lambda\expect{Z_{\sqrt{n}}}=e^{-\Omega(\sqrt{n})}$.

  Let the random variable $P_i$ be the number of lineages in family
  tree $i,1\leq i\leq \lambda$.  Similarly to the proof of
  Lemma~\ref{lemma:branching-width}, the expected number of different
  lineages in the family tree is no more than
  \begin{align*}
    \expect{P_i} \leq 
    \sum_{t=1}^{\log \lambda}\expect{Z_t} + 
    \sum_{t=\log\lambda+1}^\infty \expect{Z_t}
    \leq 
    2\lambda + \frac{\rho\lambda}{1-\rho} = O(\lambda).
  \end{align*}
  We consider the run a failure if there are more than $\lambda^3$
  lineages in all family trees. By Markov's inequality, the
  probability of this failure event is no more than
  \begin{align*}
    \prob{\sum_{i=1}^\lambda P_i\geq \lambda^3}\leq \sum_{i=1}^\lambda
    \expect{P_i} / \lambda^3 = O(1/\lambda).
  \end{align*}

  We now bound the probability that a given lineage finds an optimal
  search point. Define $\sigma':=\sigma-\delta-(k+4)/n$. To find the
  optimum, it is necessary that all the bits in the interval of length
  $\sigma'n$, starting from position $k+4$, are 1-bits.
  We consider the run a failure if any of the individuals in the
  initial population has less than $\sigma'n/3$ 0-bits in this
  interval. By a Chernoff bound and the union bound, the probability of
  this failure event is no more than $\lambda
  e^{-\Omega(n)}=e^{-\Omega(n)}$.

  The probability of flipping a given \zerobit within $\sqrt{n}$
  generations is by the union bound no more than
  $\chi/\sqrt{n}$. Hence, the probability that all of the at least
  $\sigma'n/3$ \zerobits have been flipped is less than
  $(\chi/\sqrt{n})^{\sigma'n/3}=n^{-\Omega(n)}$.  
  The probability that any of the at most $\lambda^3$ lineages finds
  the optimum within $\sqrt{n}$ generations is by the union bound 
  no more than $\lambda^3 n^{-\Omega(n)} = n^{-\Omega(n)}$.

  If none of the failure events occur, then no globally optimal search
  point has been found during the first $rn^2$ generations. The 
  probability that any of the failure events occur is by union bound
  less than $1/2+o(1)$. The proposition therefore follows.
\end{proof}

Once the equilibrium position has been reached, we will prove 
that it is hard to obtain the global optimum. We will rely on 
the fact that it is necessary to have at least $\delta n/3$ 
\zerobits in the interval from $(\sigma+\delta)n$ to
$(\sigma+2\delta)n$, and that any individual with a \zerobit 
in this interval will be ranked worse than at least half of 
the population.

\begin{proposition}\label{prop:belowcritical}
  Let $\sigma$ and $\delta$ be any constants that satisfy
  $0<\delta<\sigma<1-3\delta$.  If the Linear Ranking EA with
  population size $\lambda$, where $n\leq \lambda\leq n^k$, for any
  constant integer $k\geq 1$, with selection pressure $\eta$ and
  constant mutation rate $\chi>0$ satisfying
  $\eta > (2e^{\chi(\sigma+3\delta)}-1)/(1-\delta)$
  is applied to \selpres, and the $(1+\delta)/2$-ranked individual
  reaches at least $(\sigma+2\delta)n$ leading \onebits before the
  optimum has been found, then the probability that the optimum is
  found within $e^{cn}$ function evaluations is $e^{-\Omega(n)}$, for
  some constant $c>0$.
\end{proposition}
\begin{proof}
  Define $\gamma:=(1+\delta)/2$, and note that
  \begin{align*}
    \frac{\beta(\gamma)}{\gamma}
    & = \eta(1-\gamma)+\gamma
      = \frac{\eta(1-\delta)+1+\delta}{2}
     > e^{\chi(\sigma+3\delta)}.
  \end{align*}
  Hence, we have
  \begin{align}
    \xi^* := \ln(\beta(\gamma)/\gamma)/\chi
    & > \sigma+3\delta.\label{eq:xistar}
  \end{align}
  Let $\xi_0 := \sigma+2\delta=\xi^*-\delta$. Again, we apply the technique of
  non-selective family trees and define the \emph{core} as the set of
  search points with more than $\xi_0 n$ leading 1\nobreakdash-bits.  By
  the first part of Theorem~\ref{thm:eq-pos}, the probability that the
  $\gamma$-ranked individual has less than $\xi_0n$
  leading \onebits within $e^{cn}$ generations is $e^{-\Omega(n)}$ for
  sufficiently small $c$. If this event does happen, we say that a
  \emph{failure} has occurred.  Assuming no failure, each family tree
  member is selected in expectation less than
  $\rho<\alpha((1+\delta)/2)=1-(\eta-1)\delta<1$ times per generation.

  We first estimate the extinction time of each family tree, and the
  total number of lineages among the at most $m:=\lambda e^{cn}$
  family trees. The reproductive rate is bounded from above by a
  constant $\rho<1$. Hence, by Lemma \ref{lemma:branching-width}, the
  probability that a given family tree survives longer than
  $t:=2cn/\ln(1/\rho)$ generations is $\rho^t=e^{-2cn}$. By union
  bound, the probability that any family tree survives longer than $t$
  generations is less than $\lambda e^{-2cn}$, and we say that a
  failure has occurred if a family tree survives longer than $t$
  generations.  For each $i,$ where $1\leq i\leq m,$ let the random variable $P_i$
  denote the number of lineages in family tree $i$.  By Lemma
  \ref{lemma:branching-width} and Markov's inequality, the probability
  that the number of lineages in all the family trees exceeds
  $e^{2cn}\rho/(1-\rho)$, is
  \begin{align*}
    \prob{\sum_{i=1}^m P_i\geq \frac{e^{2cn}\rho}{1-\rho}} 
    \leq \frac{(1-\rho)\sum_{i=1}^m \expect{P_i}}{\rho e^{2cn}}
    \leq \lambda e^{-cn}.
  \end{align*}
  If this happens, we say that a failure has occurred.

  We then bound the probability that any given member of the family
  tree is optimal. To be optimal, it is necessary that there are at
  least $\delta n/3$ \zerobits in the interval from $1$ to $\xi_0 n$. We
  therefore optimistically assume that this is the case for the family
  tree member in question.  However, none of these \zerobits must
  occur in the interval from bit position $k+4$ to bit position
  $(\sigma-\delta)n$, otherwise the family tree member is not
  optimal. The length of this interval is
  $(\sigma-\delta-o(1))n=\Omega(n)$. Since the family tree is
  non-selective, the positions of these 0-bits are chosen uniformly at
  random among the $\xi_0 n$ bit positions. In particular, the
  probability of choosing a 0-bit within this interval, assuming no
  such bit position has been chosen yet, is at least
  $\Omega(n)/\xi_0 n>c'$, for some constant $c'>0$. And the 
  probability that none of the at least $\delta n/3$ \zerobits 
  are chosen from this interval is no more than $(1-c')^{\delta n/3}=e^{-\Omega(n)}$.

  There are at most $t$ family tree members per lineage.  The
  probability that any of the $te^{2cn}\rho/(1-\rho)\leq e^{3cn}$ family
  tree members is optimal is by union bound no more than
  $e^{3cn} e^{-\Omega(n)} = e^{-\Omega(n)}$, assuming that 
  $c$ is a sufficiently small constant. Taking into account 
  all the failure probabilities, the probability that the 
  optimum is found within $e^{cn}$ generations is $e^{-\Omega(n)}$,
  for a sufficiently small constant $c>0$.
\end{proof}

By combining the previous, intermediate results, we can finally prove
the main result of this section.

\begin{theorem}\label{thm:runtime-high-selpres}
  Let $\sigma$ and $\delta$ be any constants that satisfy
  $0<\delta<\sigma<1-3\delta$.  The expected runtime
  of the Linear Ranking EA with population size $\lambda,
  n\leq\lambda\leq n^k$, for any integer $k\geq 1$, and 
  selection pressure $\eta$ and constant mutation rate $\chi>0$ 
  satisfying 
  $\eta > (2e^{\chi(\sigma+3\delta)}-1)/(1-\delta)$
  is $e^{\Omega(n)}$.
\end{theorem}
\begin{proof}
  Define $\gamma:=(1+\delta)/2$ and $\xi^* := \ln(\beta(\gamma)/\gamma)/\chi$.
  By Eq. (\ref{eq:xistar}) in the proof of 
  Proposition~\ref{prop:belowcritical}, it holds that $\xi^* -\delta> \sigma+2\delta.$
  By Theorem~\ref{thm:eq-time-below} and Markov's inequality, there is 
  a constant probability that the $\gamma$-ranked 
  individual has reached at least $(\xi^*-\delta)n>(\sigma+2\delta)n$
  leading \onebits within $rn^2$ generations, for some constant $r$. 
  By Proposition~\ref{prop:prob-opt-in-n-square}, the probability
  that the optimum has not been found within the first $rn^2$ 
  generations is $\Omega(1)$. If the optimum has not been
  found before the $\gamma$-ranked individual has $(\sigma+2\delta)n$
  leading \onebits, then by Proposition~\ref{prop:belowcritical},
  the expected runtime is $e^{\Omega(n)}$. The unconditional
  expected runtime of the Linear Ranking EA is therefore 
  $e^{\Omega(n)}$.
\end{proof}

\subsection{Too Low Selection Pressure}\label{sec:too-low}

This section proves an analogue to
Theorem~\ref{thm:runtime-high-selpres} for parameter settings where
the equilibrium position $n(\ln\eta)/\chi$ is below $(\sigma -\delta)
n$.  \ie, it is shown that \selpres is also hard when the selection
pressure is too low. To prove this, it suffices to show that with
overwhelming probability, no individual reaches more than
$n\ln(\eta\kappa\phi)/\chi$ leading \onebits in exponential time, for
appropriately chosen constants $\kappa, \phi>1$.  Again, we will apply
the technique of non-selective family trees, but with a different core
than in the previous section. The core is here defined as the set of
search points with prefix sum less than $n\ln(\eta\kappa)/\chi$, where
the \emph{prefix sum} is the number of \onebits in the first
$n\ln(\eta\kappa\phi)/\chi$ bit positions of the search point.
Clearly, to obtain at least $n\ln(\eta\kappa\phi)/\chi$ leading
\onebits, it is necessary to have prefix sum exactly
$n\ln(\eta\kappa\phi)/\chi$.  We will consider individuals outside the
core, \ie, the individuals with prefix sums in the interval from
$n\ln(\eta\kappa)/\chi$ to $n\ln(\eta\kappa\phi)/\chi$. Note that
choosing $\kappa$ and $\phi$ to be constants slightly larger than 1
implies that this interval begins slightly above the equilibrium
position $n\ln(\eta)/\chi$ given by Theorem~\ref{thm:eq-pos} (see
Fig.~\ref{fig:branchingtypes}).

Single-type branching processes are not directly applicable to analyse
this drift process, because they have no way of representing how far each
family tree member is from the core. Instead, we will consider a more
detailed model based on \emph{multi-type branching processes} (see
e.g. Haccou \etal \cite{Haccou2005BranchingProcesses}). Such branching
processes generalise single-type branching processes by having
individuals of multiple types. In our application, the type of an
individual corresponds to the prefix-sum of the individual. Before
defining and studying this particular process, we will describe some
general aspects of multi-type branching processes.

\begin{definition}[Multi-Type Branching Process \cite{Haccou2005BranchingProcesses}]\label{def:multi-type-branching-process}
  A multi-type branching process with $d$ types is a Markov process
  $Z_0, Z_1, ...$ on $\mathbb{N}^d_0$, which for all $t\geq 0$, is given by
  \begin{align*}
    Z_{t+1} &:= \sum_{j=1}^d\sum_{i=1}^{Z_{tj}} \xi_i^{(j)},         
  \end{align*}
  where for all $j, 1\leq j\leq d$, $\xi_i^{(j)}\in \mathbb{N}_0^d$ are
  i.i.d. random vectors having expectation
  $\expect{\xi^{(j)}} =:\trans{(m_{j1},m_{j2},...,m_{jd})}$.    
  The associated matrix $M := (m_{jk})_{d\times d}$ is called the 
  \emph{mean matrix} of the process.
\end{definition}

Definition \ref{def:multi-type-branching-process} states that the
population vector $Z_{t+1}$ for generation $t+1$ is defined as a sum
of offspring vectors, one offspring vector for each of the individuals
in generation $t$. In particular, the vector element $Z_{tj}$ denotes
the number of individuals of type $j, 1\leq j\leq d,$ in generation
$t$. And $\xi_i^{(j)}$ denotes the offspring vector for the
$i$-th individual, $1\leq i\leq Z_{nj},$ of type $j$.  The $k$-th
element, $1\leq k\leq d,$ of this offspring vector
$\xi_{i}^{(j)}$ represents the number of offspring of type $k$ this
individual produced.

Analogously to the case of single-type branching processes, the
expectation of a multi-type branching process $Z_{t\geq 0}$ with 
mean matrix $M$ follows
\begin{align*}
  \trans{\expect{Z_{t}}}
      = \trans{\expect{\expect{Z_{t}\mid Z_{t-1}}}} 
       = \trans{\expect{Z_{t-1}}}M 
       = \trans{\expect{Z_0}}M^t.       
\end{align*}
Hence, the long-term behaviour of the branching-process depends on the
matrix power $M^t$. Calculating matrix powers can in general be
non-trivial. However, if the branching process has the property that
for any pair of types $i,j$, it is possible that a type $j$-individual
has an ancestor of type $i$, then the corresponding mean matrix is
\emph{irreducible} \cite{Seneta1973NonNegMatrices}. 

\begin{definition}[Irreducible matrix \cite{Seneta1973NonNegMatrices}]\label{def:irreducible}
  A $d\times d$ non-negative matrix $M$ is \emph{irreducible} if
  for every pair $i,j$ of its index set, there exists a positive
  integer $t$ such that $m_{ij}^{(t)}>0$, where $m_{ij}^{(t)}$ are 
  the elements of the $t$-th matrix power $M^t$.
\end{definition}
If the mean matrix $M$ is irreducible, then Theorem~\ref{thm:perron-frobenius} 
implies that the asymptotics of the matrix power $M^t$ depend on 
the largest eigenvalue of $M$.
\begin{theorem}[Perron-Frobenius \cite{Haccou2005BranchingProcesses}]\label{thm:perron-frobenius}
  If $M$ is an irreducible matrix with non-negative elements,
  then it has a unique positive eigenvalue $\rho$, called the 
  \emph{Perron root} of $M$, that is greater in
  absolute value than any other eigenvalue. All elements of the left
  and right eigenvectors $u=\trans{(u_1,...,u_d)}$ and $v=\trans{(v_1,...,v_d)}$
  that correspond to $\rho$ can be chosen positive and such that
  $\sum_{k=1}^d u_k = 1$ and $\sum_{k=1}^d u_kv_k = 1$.  In addition,
  \begin{align*}
    M^n & = \rho^n\cdot A+B^n,
  \end{align*}
  where $A=(v_iu_j)_{i,j=1}^d$ and $B$ are matrices that satisfy the
  conditions
  \begin{enumerate}
  \item $AB=BA=0$
  \item There are constants $\rho_1\in(0,\rho)$ and $C>0$ such that
    none of the elements of the matrix $B^n$ exceeds $C\rho_1^n$.
  \end{enumerate}
\end{theorem}

A central attribute of a multi-type branching process is therefore 
the Perron root of its mean matrix $M$, denoted $\rho(M)$.
A multi-type branching process with mean matrix $M$ is classified as
\emph{sub-critical} if $\rho(M)<1$, \emph{critical} if $\rho(M)=1$ and
\emph{super-critical} if $\rho(M)>1$.  Theorem
\ref{thm:perron-frobenius} implies that any sub-critical multi-type
branching process will eventually become extinct. However, to obtain
good bounds on the probability of extinction within a given number of
generations $t$ using Theorem \ref{thm:perron-frobenius}, one also has
to take into account matrix $A$ that is defined in terms of both
the left and right eigenvectors. Instead of directly applying Theorem
\ref{thm:perron-frobenius}, it will be more convenient to use the
following lemma.

\begin{lemma}[\cite{Haccou2005BranchingProcesses}]\label{lemma:branching-size-tail}
  Let $Z_0, Z_1,...$ be a multi-type branching process with
  irreducible mean matrix $M=(m_{ij})_{d\times d}$.
  If the process started with a single individual of type $h$, 
  then for any $k>0$ and $t\geq 1$,
  \begin{align*}
    \prob{\sum_{j=1}^d Z_{tj}\geq k\mid Z_{0}=e_h} & \leq \frac{\rho(M)^t}{k}\cdot \frac{v_h}{v^*},
  \end{align*}
  where $e_h, 1\leq h\leq d,$ denote the standard basis vectors,
  $\rho(M)$ is the Perron root of $M$ with the corresponding right
  eigenvector $v$, and $v^* := \min_{1\leq i\leq d} v_i$.
\end{lemma}
\begin{proof}
  The proof follows \cite[p. 122]{Haccou2005BranchingProcesses}.
  By Theorem~\ref{thm:perron-frobenius}, matrix $M$ has a unique
  largest eigenvalue $\rho(M)$, and all the elements of the
  corresponding right eigenvector $v$ are positive, implying $v^*>0$.
  The probability that the process consists of more than $k$ individuals in 
  generation $t$, conditional on the event that the process started
  with a single individual of type $h$, can be bounded as
  \begin{align*}
    \prob{\sum_{j=1}^d Z_{tj}\geq k\mid Z_0=e_h}
      & = \prob{\sum_{j=1}^d Z_{tj}v^*\geq kv^*\mid Z_0=e_h}\\
      &\leq \prob{\sum_{j=1}^d Z_{tj}v_j\geq kv^*\mid Z_0=e_h}.
  \end{align*}
  Markov's inequality and linearity of expectation give
  \begin{align*}
    \prob{\sum_{j=1}^d Z_{tj}v_j\geq kv^*\mid Z_0=e_h}
    & \leq \expect{\sum_{j=1}^d Z_{tj}v_j\mid Z_0=e_h}\cdot\frac{1}{kv^*}\\
    &  =     \sum_{j=1}^d\expect{Z_{tj}\mid Z_0=e_h}\cdot \frac{v_j}{kv^*}.
  \end{align*}
  As seen above, the expectation on the right hand side can be
  expressed as 
  \begin{align*}
    \trans{\expect{Z_{t}\mid Z_0=e_h}} = \trans{\expect{Z_0\mid Z_0=e_h}}M^t.    
  \end{align*}
  Additionally, by taking into account the starting conditions, $Z_{0h}=1$
  and $Z_{0j}=0$, for all indices $j\neq h$, this simplifies further to
  \begin{align*}
    \sum_{j=1}^d\expect{Z_{tj}\mid Z_0=e_h}\cdot \frac{v_j}{kv^*} 
     &= \sum_{j=1}^d\sum_{i=1}^d \expect{Z_{0i}\mid Z_0=e_h}\cdot m_{ij}^{(t)}\cdot \frac{v_j}{kv^*}\\
     &= \sum_{j=1}^d m_{hj}^{(t)}\cdot \frac{v_j}{kv^*}.
  \end{align*}
  Finally, by iterating 
  \begin{align*}
    M^tv=M^{t-1}(Mv)=\rho(M)\cdot M^{t-1}v,
  \end{align*}
  which on
  coordinate form gives 
  \begin{align*}
    \sum_{j=1}^d m_{hj}^{(t)}v_j=\rho(M)^t\cdot v_h,   
  \end{align*}
  one obtains the final bound
  \begin{align*}
    \prob{\sum_{j=1}^d Z_{tj}\geq k\mid Z_0=e_h}\leq \frac{\rho(M)^t}{k}\cdot \frac{v_h}{v^*}.
  \end{align*}
\end{proof}

\begin{figure}
  \centering  
  \begin{tikzpicture}

    \draw[|-latex] (0cm,0cm) -- (8.25cm,0cm);

    \draw (1.0cm, 1.0cm) node[text width=1.5cm] {\small Prefix Sum};
    \draw (1.0cm,-1.0cm) node[text width=1.5cm] {\small Branching Process Type};
 
    \draw (2.0cm,-0.1cm) -- (2.0cm,0.1cm) 
    node[above] {$\frac{n}{\chi}\ln\eta$};

    \draw (3.5cm,-0.1cm) 
    node[below] {$\frac{n}{\chi}\ln \phi$} -- (3.5cm,0.1cm)
    node[above] {$\frac{n}{\chi}\ln(\eta\kappa)$};

    \draw (7.5cm,-0.1cm) 
    node[below] {$0$} -- (7.5cm,0.1cm)
    node[above] {$\frac{n}{\chi}\ln(\eta\kappa\phi)$};

    \draw[fill] (4.5cm,0cm) node[below] {$j$} circle (2pt);
    \draw[fill] (6.5cm,0cm) node[below] {$i$} circle (2pt);
    \draw[|-latex] (6.5,0cm) .. controls (5.75cm,0.5cm) and (5.25cm,0.5cm) .. (4.6cm,0.1cm);
    \draw (5.5,0.75) node {$p_{ij}$};

  \end{tikzpicture}

  \caption{Multi-type Branching Process Model in Theorem
    \ref{thm:unlikely-above}.  The prefix-sum of an individual is the
    number of \onebits in the first $n\ln(\eta\kappa\phi)/\chi$
    bit-positions. The population core contains all individuals with
    prefix sum lower than $n\ln(\eta\kappa)/\chi$, which is slightly
    above the equilibrium value of $n\ln(\eta)/\chi$ from 
    Theorem~\ref{thm:eq-pos}. The multi-type branching process considers
    individuals outside the core, where the type of an individual is
    given by the number of \zerobits in the first
    $n\ln(\eta\kappa\phi)/\chi$ bit-positions. The probability that an
    offspring of a type $i$-individual is a type $j$-individual, is
    $p_{ij}$.  }
  \label{fig:branchingtypes}
\end{figure}
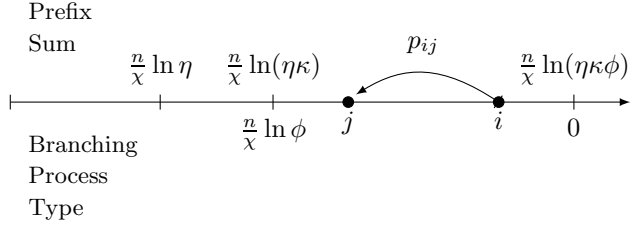

We will now describe how to model a non-selective family tree outside
the core as a multi-type branching process (see
Fig.~\ref{fig:branchingtypes}). 
Recall that the prefix sum of a search
point is the number of \onebits in the first
$n\ln(\eta\kappa\phi)/\chi$ bit positions of the search point, and
that the core is defined as all search points with prefix-sum less
than $n\ln(\eta\kappa)/\chi$ \onebits. The process has
$n(\ln\phi)/\chi$ types. A family tree member has type $i$ if its prefix
sum is $n\ln(\eta\kappa\phi)/\chi-i$. The element $a_{ij}$ of the mean
matrix $A$ of this branching process represents the expected number of
offspring a type $i$-individual gets of type $j$-individuals per
generation. Since we are looking for a lower bound on the extinction
probability, we will over-estimate the matrix elements, which can only
decrease the extinction probability.  By the definition of linear
ranking selection, the expected number of times during one generation
in which any individual is selected is no more than $\eta$. We will
therefore use $a_{ij}=\eta\cdot p_{ij}$, where $p_{ij}$ is the
probability that mutating a type $i$-individual creates a type
$j$-individual. To simplify the proof of the second part of Lemma
\ref{lemma:perron-bound}, we overestimate the probability $p_{ij}$ to
$1/n^2$ for the indices $i$ and $j$ where $j-i\geq 2\log n + 1$. Note
that the probability that none of the first $n\ln(\eta\kappa)/\chi$
bits are flipped is less than
$\exp(-\ln(\eta\kappa))=1/\eta\kappa$. In particular, this means that
$\eta\cdot p_{ii}\leq\eta/\eta\kappa=1/\kappa:=a_{ii}$.  The full
definition of the mean matrix is as follows.

\begin{definition}[Mean Matrix $A$]\label{def:selpresmeanmatrix}
  For any integer $n\geq 1$ and real numbers 
  $\eta, \chi, \phi,\kappa,\varepsilon$ where $0<\chi, 1\leq \eta$ and
  $1<\phi<\kappa\leq\varepsilon$, define the 
  $n\ln(\phi)/\chi\times n\ln(\phi)/\chi$ matrix $A=(a_{ij})$ as 
  \begin{align*}
    a_{ij} & =
    \begin{cases}
      \eta/n^2      & \text{ if }2\log n + 1\leq j-i,\\
      \eta\cdot{n\ln(\eta\kappa\phi)/\chi\choose j-i}\cdot \left(\frac{\chi}{n}\right)^{j-i} & \text{ if } 1\leq j-i\leq2\log n,\\
      1/\kappa                                        & \text{ if } i=j,\text{ and}\\
      1/\kappa\cdot{i\choose i-j}\cdot \left(\frac{\chi}{n}\right)^{i-j} & \text{ if } i>j.
    \end{cases}
  \end{align*}
\end{definition}

In order to apply Lemma~\ref{lemma:branching-size-tail} to mean matrix
$A$ defined above, we first provide upper bounds on the Perron root of
$A$ and on the maximal ratio between the elements of the corresponding
right eigenvector.

\begin{lemma}\label{lemma:perron-bound} 
  For any integer $n\geq 1$, and real numbers
  $\eta, 1< \eta\leq 2,$  
  $\chi>0$, and
  $\varepsilon>1$,
  there exist real numbers $\kappa$ and $\phi$,
  $1<\phi<\kappa\leq\varepsilon$, 
  such that matrix $A$
  given by Definition \ref{def:selpresmeanmatrix}  
  has Perron root bounded from above by $\rho(A)<c$ for some constant
  $c<1$.  Furthermore, for any $h$, $1\leq h\leq n\ln(\phi)/\chi$, the
  corresponding right eigenvector $v$, where $v^* := \min_{i} v_i$,
  satisfies
  \begin{align*}
    \frac{v_h}{v^*} & \leq 2^{n\ln(\phi)/\chi}\cdot \left(\frac{n}{\chi}\right)^{n\ln(\phi)/\chi-h}.
  \end{align*}
\end{lemma}
\begin{proof}
  Set $\kappa:=\varepsilon$.  Since $a_{ij}>0$ for all $i,j$, matrix
  $A$ is by Definition~\ref{def:irreducible} irreducible, and
  Theorem~\ref{thm:perron-frobenius} applies to the matrix. Expressing
  the matrix as $A=1/\kappa\cdot I + B$, where $B := A-1/\kappa\cdot
  I$, and $I$ is the identity matrix, the Perron root is
  $\rho(A)=1/\kappa+\rho(B)$.

  The Frobenius bound for the Perron root of a non-negative matrix $M=(m_{ij})$
  states that $\rho(M)\leq \max_j c_j(M)$
  \cite{Kolotilina2004PerronBounds}, where $c_j(M):=\sum_i m_{ij}$ is
  the $j$-th column sum of $M$.  However, when applied directly to our
  matrix, this bound is insufficient for our purposes. Instead, we can
  consider the transformation $SBS^{-1}$, for an invertible matrix
  \begin{align*}
    S:=\diag(x_1,x_2,...,x_{n\ln(\phi)/\chi}).     
  \end{align*}
  To see why this
  transformation is helpful, note that for any matrix $A$ with the
  same dimensions as $S$, we have $\det(SAS^{-1})=\det(A)$. So if
  $\rho$ is an eigenvalue of $B$, then 
  \begin{align*}
  0 & =\det(B-\rho I)\\
    & =\det(S(B-\rho I)S^{-1})\\
    & =\det(SBS^{-1}-\rho I),    
  \end{align*}
  and $\rho$ must also be an eigenvalue of $SBS^{-1}$. 
  It follows that $\rho(B)=\rho(SBS^{-1})$.
  We will therefore apply the Frobenius bound to the matrix
  $SBS^{-1}$, which has off-diagonal elements
  \begin{align*}
    (SBS^{-1})_{ij} = a_{ij}\cdot\frac{x_i}{x_j}.
  \end{align*}
  Define 
  $x_i := q^{i}$ where 
  \begin{align*}
    q := \frac{\ln(\eta\kappa\phi)}{\ln(1+1/r\eta)},    
  \end{align*}
  for some constant $r>1/(\eta-1)\geq 1$ that will be specified later. 
  Since $\eta=1+c$ for some $c>0$, the constant $q$ is
  bounded as
  \begin{align*}
    q & 
    > \frac{\ln\eta}{\ln(1+\frac{1}{r\eta})}
    > \frac{\ln\eta}{\ln(2-\frac{1}{\eta})}
    = \frac{\ln\eta}{\ln\eta+\ln(\frac{2}{\eta}-\frac{1}{\eta^2})}
    >    1.
  \end{align*}

  The sum of any column $j$ can be bounded by the three sums
  \begin{align*}
    \sum_{i=1}^{j-2\log n-1} a_{ij}\cdot\frac{x_i}{x_j} 
    & \leq n\cdot\frac{\eta}{n^2} = \frac{\eta}{n},\\
    \sum_{i=j-2\log n}^{j-1} a_{ij}\cdot\frac{x_i}{x_j} 
    & \leq \eta\cdot\sum_{i=1}^{j-1} {n\ln(\eta\kappa\phi)/\chi\choose j-i}\cdot \left(\frac{\chi}{n}\right)^{j-i}\cdot q^{i-j}\\
    & \leq \eta\cdot\sum_{i=1}^{j-1} \frac{(\ln(\eta\kappa\phi)/q)^{j-i}}{(j-i)!}\\
    & \leq \eta\cdot\sum_{k=1}^{\infty} \frac{(\ln(\eta\kappa\phi)/q)^{k}}{k!}\\
    &  =   \eta\cdot(\exp(\ln(\eta\kappa\phi)/q)-1),\quad\text{ and}\\
           \sum_{i=j+1}^{n\ln(\phi)/\chi} a_{ij}\cdot\frac{x_i}{x_j} & 
       =   \frac{1}{\kappa}\cdot\sum_{i=j+1}^{n\ln(\phi)/\chi}{i\choose i-j}\cdot \left(\frac{\chi}{n}\right)^{i-j}\cdot q^{i-j} \\
    & \leq \frac{1}{\kappa}\cdot\sum_{i=j+1}^{n\ln(\phi)/\chi}{n\ln(\phi)/\chi\choose i-j}\cdot \left(\frac{\chi}{n}\right)^{i-j}\cdot q^{i-j}\\
    & \leq \frac{1}{\kappa}\cdot\sum_{i=j+1}^{n\ln(\phi)/\chi}\frac{(q\ln\phi)^{i-j}}{(i-j)!}\\
    & \leq \frac{1}{\kappa}\cdot\sum_{k=1}^{\infty}\frac{(q\ln\phi)^k}{k!}\\
    & =    \frac{1}{\kappa}\cdot(\exp(q\ln\phi)-1).
  \end{align*}
  The Perron root of matrix $A$ can now be bounded by
  \begin{align*}
    \rho(A) 
    & \leq \frac{1}{\kappa} + 
           \max_j c_j(SBS^{-1}) \\
    & =    \frac{1}{\kappa} + 
           \max_j\sum_{i\neq j}^{n\ln(\phi)/\chi} a_{ij}\cdot\frac{x_i}{x_j} \\
    & \leq \frac{\eta}{n} +
           \eta\cdot (\exp(\ln(\eta\kappa\phi)/q)-1) +
           \frac{1}{\kappa}\cdot\exp(q \ln\phi ) \\
    & =    \frac{\eta}{n} +
           \frac{1}{r} +
           \frac{\phi^q}{\kappa}.
   \end{align*}
   Choosing $\phi$ sufficiently small, such that $1<\phi<\kappa^{1/2q}$, and
   defining the constant $r:=\frac{2}{\eta-1}\cdot\frac{\sqrt{\kappa}}{\sqrt{\kappa}-1}>1/(\eta-1)$, we have
   \begin{align*}
     \rho(A) 
     & \leq \frac{\eta}{n} + 
            \frac{1}{r} + 
            \frac{\phi^q}{\kappa} \\
     & \leq \frac{\eta}{n} + 
            \frac{\sqrt{\kappa}-1}{2\sqrt{\kappa}} + 
            \frac{1}{\sqrt{\kappa}}\\
     & =    \frac{\eta}{n} + \frac{1}{2} + 
            \frac{1}{2\sqrt{\kappa}} 
       < 1.
   \end{align*}

   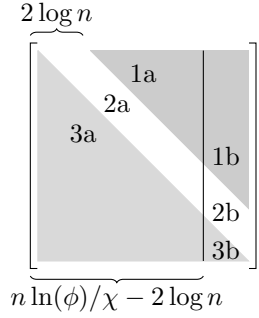
\begin{figure}
     \centering
     \begin{tikzpicture}

       \draw (0.0, 0.0) rectangle (3.0,3.0);
       \draw[white,fill=white] (0.1,-0.1) rectangle (2.9,3.1);

       \draw[fill=gray!30,gray!30] (0.1, 2.9) -- (2.9,0.1) -- (0.1,0.1) -- (0.1, 2.9);
       \draw[fill=gray!40,gray!40] (0.8, 2.9) -- (2.9,0.8) -- (2.9,2.9) -- (0.8, 2.9);

       \draw (2.3, 0.1) -- (2.3,2.9);

       \draw[snake=brace] (0.0, 3.1) -- node[above] {$2\log n$} (0.7,3.1);
       \draw[snake=brace] (2.3,-0.1) -- node[below] {$n\ln(\phi)/\chi-2\log n$} (0.0,-0.1);

       \node at (1.50, 2.60) {1a};
       \node at (1.15, 2.20) {2a};
       \node at (0.70, 1.80) {3a};

       \node at (2.60, 1.50) {1b};
       \node at (2.60, 0.75) {2b};
       \node at (2.60, 0.25) {3b};

     \end{tikzpicture}

     \caption{Structure of matrix $A$ in Definition \ref{def:selpresmeanmatrix}.}
     \label{fig:matrix-structure}
   \end{figure}

   The second part of the lemma involves for any $h$, to bound the
   ratio $v_h/v^*$ where $v$ is the right eigenvector corresponding to
   the eigenvalue $\rho$. In the special case where the index $h$
   corresponds to the eigenvector element with largest value, this
   ratio is called the \emph{principal ratio}.  By generalising Minc's
   bound for the principal ratio \cite{Minc1970Eigenvector}, one
   obtains the upper bound
   \begin{align*}
     \frac{v_h}{v^*}
     = \max_k\frac{v_h}{v_k} 
     = \max_k\frac{\rho v_h}{\rho v_k}
     = \max_k\frac{\sum_{j} a_{hj}\cdot v_j}{\sum_{j} a_{kj}\cdot v_j}
     \leq \max_{k,j}\frac{a_{hj}}{a_{kj}}.
   \end{align*}
   It now suffices to prove that the matrix elements of $A$ satisfy
   \begin{align*}
     \forall h,j,k\quad  \frac{a_{hj}}{a_{kj}} 
     & \leq 2^{n\ln(\phi)/\chi}\cdot \left(\frac{n}{\chi}\right)^{n\ln(\phi)/\chi-h}.
   \end{align*}

   To prove that these inequalities hold, we first find a
   lower bound $a^*_j$ on the minimal element
   along any column, \ie $\min_k a_{kj}\geq a^*_j,$ for any column index $j$. 
   As illustrated in Fig.~\ref{fig:matrix-structure},
   the matrix elements of $A$ can be divided into six cases according
   to their column and row indices, 
   For case 1a and 1b, where $2\log n +1\leq j-k\leq n\ln(\phi)/\chi$,
   \begin{align*}
     a_{kj} & > \frac{1}{n^2}.
   \end{align*}
   For case 2a and 2b, where $0< j-k\leq 2\log n$,
   \begin{align*}
     a_{kj} & 
     \geq \left(\frac{\chi}{n}\right)^{j-k} 
     \geq \left(\frac{\chi}{n}\right)^{2\log n}. 
   \end{align*}
   For case 3a and 3b, where $k\geq j$,
   \begin{align*}
     a_{kj}  & 
     \geq \frac{1}{\kappa}\left(\frac{\chi}{n}\right)^{k-j}
     \geq \frac{1}{\kappa}\left(\frac{\chi}{n}\right)^{n\ln(\phi)/\chi-j}.
   \end{align*}
   Hence, we can use the lower bound
   \begin{align*}
     a^*_j :=
     &
     \begin{cases}
       \frac{1}{\kappa}\left(\frac{\chi}{n}\right)^{n\ln(\phi)/\chi-j} & \text{if } j\leq n\ln(\phi)/\chi-2\log n,\text{ and}\\
       \left(\frac{\chi}{n}\right)^{2\log n}         & \text{otherwise.}
     \end{cases}
   \end{align*}

   We then upper bound the ratio $a_{hj}/a^*_j$ for all column indices $j$.
   All elements of the matrix satisfy $a_{hj}\leq \eta$. Therefore, 
   in case 1b, 2b and 3b, where $j> n\ln(\phi)/\chi-2\log n$,
   \begin{align*}
     \frac{a_{hj}}{a^*_j} 
     & \leq \eta\left(\frac{n}{\chi}\right)^{2\log n}.
   \end{align*}
   In case 1a and 2a, where $h< j\leq n\ln(\phi)/\chi-2\log n$,
   \begin{align*}
     \frac{a_{hj}}{a^*_j}
     & \leq \kappa\eta\left(\frac{n}{\chi}\right)^{n\ln(\phi)/\chi-j}
       \leq \kappa\eta\left(\frac{n}{\chi}\right)^{n\ln(\phi)/\chi-h}.
     \end{align*}
     Finally, in case 3a, where $j\leq h$ and $j\leq n\ln(\phi)/\chi-2\log n$,
     \begin{align*}
       \frac{a_{hj}}{a^*_j}
       & \leq \frac{1}{\kappa}{h\choose h-j}\cdot\left(\frac{\chi}{n}\right)^{h-j}\cdot\kappa\left(\frac{n}{\chi}\right)^{n\ln(\phi)/\chi-j}\\
       & \leq 2^{n\ln(\phi)/\chi}\cdot\left(\frac{n}{\chi}\right)^{n\ln(\phi)/\chi-h}.
     \end{align*}
     The second part of the lemma therefore holds.
\end{proof}

Having all the ingredients required to apply 
Lemma~\ref{lemma:branching-size-tail} to the mean matrix in Definition
\ref{def:selpresmeanmatrix}, we are now ready to prove the main
technical result of this section. Note that this result implies that 
Conjecture 1 in \cite{Lehre2009FOGA} holds.

\begin{theorem}\label{thm:unlikely-above}
  For any positive constant $\epsilon$, and some positive constant $c$,
  the probability that during $e^{cn}$
  generations, Linear Ranking EA with population size
  $\lambda=poly(n)$, selection pressure $\eta$, and mutation rate
  $\chi/n$, there exists any individual with at least
  ${n((\ln\eta)/\chi+\epsilon)}$
  leading \onebits is $e^{-\Omega(n)}$.
\end{theorem}
\begin{proof}
  In the following, $\kappa$ and $\phi$ are two constants such that
  $(\ln\kappa+\ln\phi)/\chi=\epsilon$, where the relative magnitudes 
  of $\kappa$ and $\phi$ are as given in the proof of 
  Lemma~\ref{lemma:perron-bound}. 

  Let the \emph{prefix sum} of a search point be the number of
  \onebits in the first $n\ln(\eta\kappa\phi)/\chi$ bits.  We will
  apply the technique of non-selective family trees, where the core is
  defined as the set of search points with prefix sum less than
  $n\ln(\eta\kappa)/\chi$ 1\nobreakdash-bits.  Clearly, any
  non-optimal individual in the core has fitness lower than
  $n\ln(\eta\kappa)/\chi$.  

  To estimate the extinction time of a given family tree, we consider
  the multi-type branching process $Z_0,Z_1,...$ having
  $n\ln(\phi)/\chi$ types, and where the mean matrix $A$ is given by
  Definition \ref{def:selpresmeanmatrix}. Let the random variable $S_t
  :=\sum_{i=1}^{n\ln(\phi)/\chi} Z_{ti}$ be the family size in
  generation $t$.  By Lemma~\ref{lemma:branching-size-tail} and
  Lemma~\ref{lemma:perron-bound}, it is clear that the extinction
  probability of the family tree depends on the type of the root of
  the family tree. The higher the prefix sum of the family root, the
  lower the extinction probability. The parent of the root of the
  family tree has prefix sum lower than $n\ln(\eta\kappa)/\chi$, 
  hence the probability that the root of the family tree has 
  type $h$, is
  \begin{align*}
    \prob{Z_0=e_h} & \leq
         {n\ln(\phi)/\chi\choose n\ln(\phi)/\chi-h}
        \cdot \left(\frac{\chi}{n}\right)^{n\ln(\phi)/\chi-h}.
  \end{align*}

  By Lemma \ref{lemma:branching-size-tail} and Lemma \ref{lemma:perron-bound},
  the probability that the family tree has more than $k$ members in generation
  $t$ is for sufficiently large $n$ and sufficiently small $\phi$ bounded by
  \begin{align*}
    \lefteqn{\prob{S_{t}\geq k}}\\
    & =    \sum_{h=1}^{n\ln(\phi)/\chi}  
           \prob{Z_0=e_h}
           \cdot \prob{\sum_{j=1}^{n\ln(\phi)/\chi} Z_{tj}\geq k \mid Z_0=e_h}\\
    & \leq \sum_{h=1}^{n\ln(\phi)/\chi}
           {n\ln(\phi)/\chi\choose n\ln(\phi)/\chi-h}
           \cdot \left(\frac{\chi}{n}\right)^{n\ln(\phi)/\chi-h}
           \cdot \frac{\rho(A)^t}{k}
           \cdot \frac{v_h}{v^*}\\
    & \leq 2^{n\ln(\phi)/\chi}
           \cdot \frac{\rho(A)^t}{k}
           \cdot \sum_{h=0}^{n\ln(\phi)/\chi}
           {n\ln(\phi)/\chi\choose h}\\
    & =    2^{2n\ln(\phi)/\chi} 
           \cdot \frac{\rho(A)^t}{k}.
  \end{align*}
  By Lemma~\ref{lemma:perron-bound}, the Perron root of matrix $A$ is
  bounded from above by a constant $\rho(A)<1$.
  Hence, for any constant $w>0$, the constant $\phi$ can be chosen
  sufficiently small such that for large $n$, the probability is
  bounded by $\prob{S_{t}\geq k} \leq \rho(A)^{t-wn}/k$.

  For $k=1$ and $w<1$,
  the probability that the non-selective family tree is not extinct in
  $n$ generations, \ie, that the \emph{height} of the tree is larger
  than $n$, is $\rho(A)^{\Omega(n)}=e^{-\Omega(n)}$. Furthermore, the
  probability that the \emph{width} of the non-selective family tree
  exceeds $k=\rho(A)^{-2wn}$ in any generation is by union bound 
  less than $n\rho(A)^{wn}=e^{-\Omega(n)}$.

  We now consider a phase of $e^{cn}$ generations. The number of
  family trees outside the core during this period is less than 
  $\lambda e^{cn}$. 
  The probability that any of these family trees survives longer 
  than $n$ generations, or are wider than $\rho(A)^{-2wn}$, is by 
  union bound 
  $\lambda e^{cn}\cdot (e^{-\Omega(n)}+e^{-\Omega(n)})=e^{-\Omega(n)}$ 
  for a 
  sufficiently small constant $c$. The number of paths from root to
  leaf within a single family tree is bounded by the product of the
  height and the width of the family tree. Hence, the expected 
  number of different paths from root to leaf in all family trees 
  is less than $\lambda e^{cn}n\rho(A)^{-2wn}$.
  The probability that it exceeds $e^{2cn}\rho(A)^{-2wn}$ is by 
  Markov's inequality $\lambda e^{cn}ne^{-2cn}=e^{-\Omega(n)}$.

  The parent of the root of each family tree has prefix sum no 
  larger than $n\ln(\eta\kappa)/\chi$. In order to reach at least
  $n\ln(\eta\kappa\phi)/\chi$ leading 1\nobreakdash-bits, it is therefore
  necessary to flip $n\ln(\phi)/\chi$  0\nobreakdash-bits within
  $n$ generations. The probability that a given \zerobit is not
  flipped during $n$ generations is $(1-\chi/n)^{n}\geq p$ for 
  some constant $p>0$. Hence, the probability that all of the 
  $n\ln(\phi)/\chi$ \zerobits are
  flipped at least once within $n$ generations is no more than
  $p^{n\ln(\phi)/\chi}=e^{-c'n}$ for some constant $c'>0$.
  Hence, by union bound, the probability that any of the
  paths attains at least $\ln(\eta\kappa\phi)/\chi$ leading 
  \onebits is less than $e^{2cn}\rho(A)^{-2wn}e^{-c'n}=e^{-\Omega(n)}$ 
  for sufficiently small $c$ and $w$.
\end{proof}

Using Theorem \ref{thm:unlikely-above}, it is now straightforward to
prove that \selpres is hard for the Linear Ranking EA when the ratio
between the selection pressure $\eta$ and the mutation rate $\chi$ is
too small.

\begin{corollary}\label{cor:runtime-low-selpres}
  The probability that Linear Ranking EA with population size
  $\lambda=\poly(n)$, bit-wise mutation rate $\chi/n$, and selection
  pressure $\eta$ satisfying $\eta<
  \exp(\chi(\sigma-\delta))-\epsilon$ for any $\epsilon>0$, finds the
  optimum of \selpres within $e^{cn}$ function evaluations is
  $e^{-\Omega(n)}$, for some constant $c>0$.
\end{corollary}
\begin{proof}
  In order to reach the optimum, it is necessary to obtain an
  individual having at least $n(\alpha-\delta)$ leading 
  1\nobreakdash-bits.  However, by Theorem~\ref{thm:unlikely-above}, 
  the probability that this happens within $e^{cn}$ generations 
  is $e^{-\Omega(n)}$ for some constant $c>0$.
\end{proof}

\section{Conclusion}

The aim of this paper has been to better understand the
relationship between mutation and selection in EAs, and in particular
to what degree this relationship can have an impact on the runtime. To
this end, we have rigorously analysed the runtime of a non-elitist
population-based EA that uses linear ranking
selection and bit-wise mutation on a family of fitness functions. We
have focused on two parameters of the EA, $\eta$ which
controls the selection pressure, and $\chi$ which controls the 
bit-wise mutation rate.

The theoretical results show that there exist fitness functions where
the parameter settings of selection pressure $\eta$ and mutation rate
$\chi$ have a dramatic impact on the runtime.  To achieve polynomial
runtime on the problem, the settings of these parameters need to be
within a narrow critical region of the parameter space, as illustrated
in Fig.~\ref{fig:finalresult}. An arbitrarily small increase in
the mutation rate, or decrease in the selection pressure can increase the
runtime of the EA from a small polynomial (\ie highly efficient), to
exponential (\ie highly inefficient). The critical factor which
determines whether the EA is efficient on the problem is not 
individual parameter settings of $\eta$ or $\chi$, but rather the
ratio between these two parameters. A too high mutation rate
$\chi$ can be balanced by increasing the selection pressure $\eta$,
and a too low selection pressure $\eta$ can be balanced by decreasing
the mutation rate $\chi$.  Furthermore, the results show that the EA
will also have exponential runtime if the selection pressure becomes
too high, or the mutation rate becomes too low. It is pointed out
that the position of the critical region in the parameter space in which
the EA is efficient is problem dependent. Hence, the EA may be
efficient with a given mutation rate and selection pressure on one
problem, but be highly inefficient with the same parameter settings on
another problem. There is therefore no balance between selection
and mutation that is good on all problems. The results shed some 
light on the possible reasons for the difficulty of parameter tuning in
practical applications of EAs.  The optimal parameter settings can be
problem dependent, and very small changes in the parameter settings
can have big impacts on the efficiency of the algorithm.

Informally, the results for the functions studied here can be
explained by the occurrence of an equilibrium state into which the
non-elitist population enters after a certain time. In this state, the
EA makes no further progress, even though there is a fitness gradient
in the search space. The position in the search space in which the
equilibrium state occurs depends on the mutation rate and the
selection pressure. When the number of new good individuals added to
the population by selection equals the number of good individuals
destroyed by mutation, then the population makes no further
progress. If the equilibrium state occurs close to the global optimum,
then the EA is efficient. If the equilibrium state occurs far from the
global optimum, then the EA is inefficient.  The results are
theoretically significant because the impact of the selection-mutation
interaction on the runtime of EAs has not previously been
analysed. Furthermore, there exist few results on the runtime of
population-based EAs, in particular those that employ both a parent
and an offspring population. Our analysis answers a challenge by Happ
\etal \cite{Happ2008Selection}, to analyse a population-based EA using
a non-elitist selection mechanism.  Although this paper analyses
selection and mutation on the surface, it actually touches upon a far
more fundamental issue of the trade-off between exploration (driven by
mutation) and exploitation (driven by selection). The analysis
presented here could potentially by used to study rigorously the
crucial issue of balancing exploration and exploitation in
evolutionary search.

In addition to the theoretical results, this paper has also introduced
some new analytical techniques to the analysis of evolutionary
algorithms. In particular, the behaviour of the main part of the
population and stray individuals are analysed separately. The analysis
of stray individuals is achieved using a concept which we call
non-selective family trees, which are then analysed as single- and
multi-type branching processes.  Furthermore, we apply the drift
theorem in two dimensions, which is not commonplace.  As already
demonstrated in \cite{Lehre2010PopNegDrift}, these new techniques are
applicable to a wide range of EAs and fitness functions.

A challenge for future experimental work is to design and analyse
strategies for dynamically adjusting the mutation rate and selection
pressure. Can self-adaptive EAs be robust on problems like those that
are described in this paper? For future theoretical work,
it would be interesting to extend the analysis to other problem
classes, to other selection mechanisms, and to EAs that 
use a crossover operator.

\section*{Acknowledgements}

The authors would like to thank Tianshi Chen for discussions
about selection mechanisms in evolutionary algorithms and Roy Thomas,
Lily Kolotilina and Jon Rowe for discussions about Perron root
bounding techniques.

\bibliographystyle{plain}

\end{document}